\documentclass[12pt]{article}

\usepackage[margin=3cm]{geometry}
\usepackage[utf8]{inputenc}
\usepackage{amsmath}
\usepackage{mathrsfs}
\usepackage{amsfonts}
\usepackage{amssymb}
\usepackage{graphicx}
\usepackage{verbatim}
\usepackage{etex}
\usepackage{tikz}
\usetikzlibrary{arrows,patterns,topaths}
\usetikzlibrary{positioning}
\tikzstyle{call} = [->, line width=1mm, lightgray]
\tikzstyle{callsequencelabel} = [black, font=\footnotesize]
\usepackage{array}
\usepackage{float}
\usepackage{booktabs}

\usepackage{listings}
\definecolor{dkgreen}{rgb}{0,.3,0}
\definecolor{dkblue}{rgb}{0,0,.3}
\definecolor{lightgrey}{gray}{0.4}
\definecolor{gray}{gray}{0.50}
\usepackage[thmmarks]{ntheorem}
\theorembodyfont{\rmfamily}
\theoremsymbol{\ensuremath{\dashv}}
\usepackage{newproof}
\newproof{proofWithManualQED}{Proof}{}

\newtheorem{theorem}{Theorem}
\newtheorem{proposition}[theorem]{Proposition}
\newtheorem{lemma}[theorem]{Lemma}
\newtheorem{corollary}[theorem]{Corollary}

\newtheorem{definition}[theorem]{Definition}
\newtheorem{example}[theorem]{Example}
\newtheorem{observation}[theorem]{Observation}

\usepackage[pdfusetitle,hidelinks]{hyperref}
\usepackage[capitalize,nameinlink,noabbrev]{cleveref}

\newcommand{\LNS}{\mathsf{LNS}}

\newcommand{\ANY}{\mathsf{ANY}}
\newcommand{\prot}{\mathsf{P}}
\newcommand{\CMO}{\mathsf{CMO}}
\newcommand{\PIG}{\mathsf{PIG}}
\renewcommand{\phi}{\varphi}
\newcommand{\asynch}{\sim}
\newcommand{\synch}{\approx}

\newcommand{\Exp}{\mathit{Exp}}

\newcommand{\eq}{\leftrightarrow}

\newcommand{\imp}{\rightarrow}

\newcommand{\et}{\wedge}
\newcommand{\vel}{\vee}
\newcommand{\Et}{\bigwedge}
\newcommand{\Vel}{\bigvee}

\newcommand{\is}{\exists}

\newcommand{\M}{\hat{K}}
\renewcommand{\phi}{\varphi}
\newcommand{\union}{\cup}
\newcommand{\Union}{\bigcup}

\newcommand{\call}[1]{\stackrel{#1}{\to}}

\newcommand{\lang}{\mathcal{L}}
\newcommand{\lbr}{[\![}
\newcommand{\rbr}{]\!]}
\newcommand{\I}[1]{\lbr #1 \rbr}
\newcommand{\Naturals}{\mathbb{N}}

\newcommand{\bigO}{\mathcal{O}}

\newcommand{\Skip}{\ensuremath{\mathit{skip}}}

\begin{document}

\title{Everyone knows that everyone knows: gossip protocols for super experts}
\author{Hans van Ditmarsch\thanks{Open University of the Netherlands; {\tt hans.vanditmarsch@ou.nl}}
  \and Malvin Gattinger\thanks{University of Amsterdam, the Netherlands; {\tt malvin@w4eg.eu}}
  \and Rahim Ramezanian\thanks{Shomara LLC, Tehran, Iran; {\tt rahim.ramezanian@gmail.com}}
}
\maketitle

\begin{abstract}
 A gossip protocol is a procedure for sharing secrets in a network. 
 The basic action in a gossip protocol is a pairwise message exchange (telephone call) wherein the calling agents exchange all the secrets they know. 
 An agent who knows all secrets is an expert. 
 The usual termination condition is that all agents are experts. 
 Instead, we explore protocols wherein the termination condition is that all agents know that all agents are experts. 
 We call such agents super experts. 
 We also investigate gossip protocols that are common knowledge among the agents. 
 Additionally, we model that agents who are super experts do not make and do not answer calls, and that this is common knowledge. 
 We investigate conditions under which protocols terminate, both in the synchronous case, where there is a global clock, and in the asynchronous case, where there is not. 
 We show that a commonly known protocol with engaged agents may terminate faster than the same commonly known protocol without engaged agents. 
\end{abstract}

\section{Introduction}\label{section:introduction}

The gossip problem addresses how to spread secrets among a group of agents by pairwise message exchanges, traditionally named telephone calls. 
We assume that each agent holds a single secret, and that when calling each other the agents exchange all the secrets they know. An agent may call another agent if it has that agent's telephone number. It is typically assumed that the goal of the information dissemination is that all agents know all secrets. 
The situation can be represented by a graph or network where the nodes are the agents and where, when two nodes are linked, the agents can call each other. They are then often called neighbours.

\paragraph*{Survey of related work}
There are many variations of the problem. 
It goes back to the early 1970s~\cite{BAKER1972191,tijdeman:1971,knoedel:1975,boydSteele1979,west82a}. 
In this classic setting (for an overview, see~\cite{hedetniemietal:1988}) only secrets are exchanged, and the focus is on minimum execution length of protocols executed by a central scheduler.
Variations (described in depth in~\cite{hedetniemietal:1988}) involve distinguishing message exchange (push-pull), from only the caller informing the person called (push), and from only the person called informing the caller (pull). Here, we only consider exchange.
Parallel calls, but where an agent cannot be involved in two parallel calls at the same time, were considered in~\cite{knoedel:1975}. We only consider sequential calls in this paper.
Different network topologies restrict what agents are neighbours: trees, lines, circles, etcetera. We only consider that all agents can call each other.

Later publications assume that the scheduling is \emph{distributed}~\cite{kermarrecetal:2007,eugster}, which means that agents calling each other are not assigned to do so by the central scheduler, but operate autonomously. The two scheduling approaches are related: in distributed gossip the role of the central scheduler is restricted to a random move of nature, determining which agent is next to call. In the traditional gossip community distributed gossip focusses on rounds of calls, where all agents simultaneous place a call --- this requires a mechanism for agents to receive multiple incoming calls from different callers~\cite{kermarrecetal:2007}. In the more recent modal logical gossip community an elegant connection between distributed and non-distributed scheduling was first proposed in~\cite{apt-tark}. With distributed protocols and random scheduling came the issue of the expected termination of gossip protocols, where nearly all approaches concur in finding the Coupon Collector's complexity (or related) of $O(n\log n)$~\cite{boydSteele1979,haigh81,moon72,DitmarschKS17}, for $n$ agents, exceptions granted for networks with restricted neighbour relations where $O(n \log^2 n)$ has been obtained\cite{Haeupler15}.

Fairly recent developments focus on gossip protocols with \emph{epistemic} preconditions for calls~\cite{apt-tark,attamahetal.ecai:2014,Attamah2017,AptW17,AptW18,CooperHMMR19,abs-1907-09097,logicofgossiping:2020}. 
For example, an agent may only call another agent once, or only if she does not know the other agent's secret. Epistemic preconditions formalize what agents know about secrets. To determine what an agent knows, it is required to reason about indistinguishable call sequences. Novel issues then come into play, such as whether callers only know what set of secrets they hold after a call (merge-then-inspect) or whether they also know what the other caller contributed to that output (inspect-then-merge), which is more informative~\cite{apt-tark,attamahetal.ecai:2014,DitmarschEPRS17,AptGH18}.  Instead of preconditions involving individual knowledge one can also consider preconditions involving common knowledge between the two agents involved in a call~\cite{AptW17,abs-1907-09097}. The presence of a global clock (synchronous gossip) provides much more information --- and thus knowledge --- than when it is absent, as in truly distributed systems (asynchronous gossip)~\cite{apt-tark,AptW18,DitmarschEPRS17}. In our setting we will also derive very different results for synchrony and for asynchrony.

In \emph{dynamic} gossip~\cite{DitmarschEPRS17,hvdetal.dynamicgossip:2019} the agents do not only exchange all the secrets they know but also all the neighbours (telephone numbers) they know.
This results in network expansion: not only the secret relation but also the number relation is expanded after a call.
The network is then dynamic, which explains the term.
If the number relation is a complete digraph (the universal relation), i.e., when all agents know all telephone numbers, then the dynamic and classic gossip problem coincide.
We assume complete digraphs and thus do not investigate dynamic gossip.

Another way in which gossip protocols can be epistemic is when they strive to realize higher-order epistemic goals. Such protocols were investigated in~\cite{HerzigM17,CooperHMMR19}. They show that arbitrarily higher-order mutual knowledge of all secrets can be obtained when the callers are permitted not merely to exchange secrets but also \emph{knowledge about secrets}. In their approach, primarily, in a call the two agents may exchange all the secrets they know.
But once this is done, they may also exchange the information `everyone knows all the secrets'. This requires that the number of agents is known.  
And once \emph{that} is done, they may exchange the information `everyone knows that everyone knows all the secrets', and so on. They thus achieve higher-order mutual knowledge of all secrets (all the agents know that all the agents know, etc.). Surprisingly, arbitrary higher-order goals cannot be reached if agents are only allowed to exchange secrets, assuming that agents can only observe their own calls~\cite{DitmarschG22}: `everyone knows that everyone knows all the secrets' is all you can get. This makes gossip protocols with that epistemic goal a particularly interesting target for analysis.

\paragraph*{Our contribution}
In this contribution we investigate gossip protocols with the epistemic goal that all agents know that all agents know all secrets. Clearly, this assumes that the agents know how many (other) agents there are.
\begin{itemize}
\item The protocol terminates when \emph{everyone knows that everyone knows all secrets}. 
\end{itemize}
However, we continue to assume that agents only exchange the same basic information as in the classic gossip problem, i.e.\ only secrets. So, unlike~\cite{HerzigM17} we do not achieve the epistemic goal by loading the messages with epistemic features. The agents may also have knowledge of the protocol, or of the behaviour of other agents. We consider various such modifications, and will investigate how making such assumptions affect properties such as termination and execution length.
\begin{itemize}
\item Agents know what gossip protocol is used by all agents.
\item Agents who know that everyone knows all secrets \emph{no longer make calls}. 
\item Agents who know that everyone knows all secrets \emph{no longer answer calls}. 
\end{itemize}
An agent who knows all secrets is called an \emph{expert}, as usual. We call an agent who knows that everyone is an expert a \emph{super expert}. So our epistemic goal is for all agents to become super experts, where we will also investigate the effect of additional assumptions such as knowledge of the protocol and that super experts no longer make and answer calls.

\paragraph*{Examples}
In the remainder of this introductory section we give examples to motivate our approach and we outline our results. 

\medskip

Let there be four agents $a,b,c,d$. Each agent holds a single secret to share.
Consider the call sequence $ab;cd;ac;bd$. In a call, agents exchange all secrets they know. After the call $ab$, agents $a$ and $b$ both know two secrets, and similarly after the call $cd$, agents $c$ and $d$ both know two secrets. Therefore, after the subsequent call $ac$, agents $a$ and $c$ both know all four secrets: they are experts. 
Similarly, after the final call $bd$, $b$ and $d$ are experts. So, after $ab;cd;ac;bd$, all agents are experts. See Table~\ref{tab:ab;cd;ac;bd}.

\begin{table}
  \centering
  \begin{tabular}{l|l@{}l@{}l@{}l l@{}l@{}l@{}l l@{}l@{}l@{}l l@{}l@{}l@{}l}
    \           & a& & &  &  &b& &  &  & &c&  &  & & &d \\
    $\call{ab}$ & a&b& &  & a&b& &  &  & &c&  &  & & &d \\
    $\call{cd}$ & a&b& &  & a&b& &  &  & &c&d &  & &c&d \\
    $\call{ac}$ & a&b&c&d & a&b& &  & a&b&c&d &  & &c&d \\
    $\call{bd}$ & a&b&c&d & a&b&c&d & a&b&c&d & a&b&c&d \\
  \end{tabular}
  \caption{Results of the call sequence $ab;cd;ac;bd$.}\label{tab:ab;cd;ac;bd}
\end{table}

In fact, the agents know a bit more than that. 
After call $ac$ agent $a$ is not only herself an expert but she also knows that agent $c$ is an expert, and agent $c$ also knows that agent $a$ is an expert. 
(We typically use female pronouns to refer to $a$, male pronouns to refer to $b$, female pronouns to refer to $c$, and so on.) Similarly, after call $bd$, agent $b$ also knows that $d$ is an expert, and $d$ also knows that $b$ is an expert. 
Can the agents continue calling each other until they all know that they are all experts, i.e., until they all know that they all know all secrets? Yes, they can.

Let us first consider agent $a$. In order to get to know that everyone knows all secrets, $a$ has to make two further calls: $ab$ and $ad$. 
Let us suppose these calls are made, and in that order, i.e.\ consider the whole sequence $ab;cd;ac;bd;ab;ad$.
First, note that before and after those calls the agents involved are already experts, so no factual information is exchanged. However, the agents still learn about each other that they are experts. Hence, after $ab$, agent $a$ knows that $b$ is an expert and after $ad$ she knows that $d$ is an expert. 
As she also knows this about herself, $a$ therefore now knows that everyone is an expert. She has become a super expert.

Let us now consider agent $b$. In call $bd$ he learnt that $d$ is an expert, and in the additional call $ab$ he learnt that $a$ is an expert. 
And again he obviously knows about himself that he is an expert. 
Therefore, in order to get to know that everyone is an expert, $b$ only needs to make one additional call, $bc$, and $b$ then is a super expert. 

We now consider agent $c$. Similarly, after yet another call $cd$, $c$ is a super expert, which can be observed by highlighting the calls wherein $c$ learns that another agent is an expert, as follows: $ab;cd;\pmb{ac};bd;ab;ad;\pmb{bc};\pmb{cd}$. 
We killed two birds with one stone, because after that final call $cd$ also agent $d$ knows that all agents are experts: $ab;cd;ac;\pmb{bd};ab;\pmb{ad};bc;\pmb{cd}$. 

Therefore, all agents are super experts after the call sequence \[ ab;cd;ac;bd;ab;ad;bc;cd.\] 

\paragraph{Motivation for missed call semantics} We now motivate our modifications of the usual call rules in gossip. We call an agent who no longer makes calls and no longer answers calls an \emph{engaged} agent, and a call that is not answered we name a \emph{missed call}.  As a first idea, suppose that agents become engaged once they are experts (so not super experts). 
Given this, can everyone still become an expert? Yes. 
For example, after the already mentioned call sequence $ab;cd;ac;bd$ all agents are experts, and all calls were answered. 
However, now consider the sequence $ab;ac;ad$. 
After this, agents $a$ and $d$ are experts. 
Agents $b$ and $c$ can now no longer become experts: if either were to call $a$ or $d$, this would be a missed call. 
Note that agents do not learn any secrets from a missed call. 
Hence in this case $b$ and $c$ can never learn the secret of $d$: they can still call each other, and after additional call $bc$ or $cb$ agents $b$ and $c$ would both know three secrets but not all four secrets, hence they are not experts. 
The protocol cannot terminate. 

We could additionally assume common knowledge among the agents that a missed call means that the agent not answering the call is an expert. But that does not make a big difference. After a missed call as above agents $b$ and $c$ would thus know that $a$ and $d$ are experts. But, for example, that agent $b$ knows that $a$ knows the secret of $d$, does not make $b$ himself know the secret of $d$. They cannot use that knowledge to become experts themselves. With the classic gossip goal wherein all agents become experts the presence of engaged agents prevents termination even for very simple protocols. We conclude that this first idea of a condition for missed calls is not very satisfactory. 

In this contribution we therefore employ the idea of missed calls in a different way. Let us now suppose that the goal of the protocol is for all agents to become super experts, and that an agent \emph{who is a super expert} no longer makes calls and no longer answers calls. This requirement is harder to fulfil than the previous requirement that an agent \emph{who is an expert} stops making and answering calls. 

We can already satisfy the stronger termination requirement that all are super experts without such missed calls, for example, with the sequence $ab;cd;ac;bd;ab;ad;bc;cd$ above. This is not entirely obvious. However, observe that after the subsequence $ab;cd;ac;bd;ab;ad$ \emph{only} agent $a$ knows that everyone is an expert, and in the subsequent call $bc$ \emph{only} agent $b$ learns that, and \emph{only} in the final call $cd$ agents $c$ and $d$ simultaneously learn that. No call is made to a super expert. Therefore, there are no missed calls.

However, now consider the call sequence $ab;cd;ac;bd;ab;ad;ba;ca;da$ with this missed call semantics.  
All final three calls are missed calls, because $a$ already knows that everyone is an expert. 
What do $b$, $c$, and $d$ respectively learn from these calls?
Well, nothing whatsoever, as just like above we did not make any assumptions so far about the meaning of a missed call in this new context. 
Therefore, after those calls we can still make the additional calls $bc;cd$ in order to satisfy that everyone knows that everyone is an expert. 

Let us now, as above, additionally assume that it is common knowledge among the agents that a missed call means that the agent not answering the call is a super expert. 
Now, unlike above, that makes a big difference. 
Given the sequence $ab;cd;ac;bd;ab;ad;ba;ca;da$, in the three final missed calls $ba$, $ca$, and $da$, respectively, agents $b$, $c$, $d$ then learn from $a$ that all agents are experts, so that after the entire sequence all agents know that all agents are experts. 
Again, we are done. 

Before we continue, let us make three more observations. 
Firstly, if the three missed calls had been ordinary calls, the termination condition would not yet have been met. 
For example, agent $d$ would then not know that agent $c$ knows all secrets. 
Additional calls would have been needed. 
Secondly, although the sequence with three missed calls is one call longer than the previous sequence that also realizes the knowledge objective, in general there are terminating sequences with missed calls that are shorter than any other terminating sequence without missed calls, as we will prove later. Thirdly, as in a missed call the agent calling must already be an expert (otherwise the agent called cannot be a super expert), no factual information would have been exchanged if that call had been an ordinary call. So the presence of missed calls does not prevent agents from becoming experts in the first place, which would have wrecked our chances to reach the protocol goal.

The modelling solution for missed calls, that is novel, is similar to a modelling solution for making protocols common knowledge, presented in~\cite{DitmarschGKP19}. 
We incorporated both in this contribution. 
This also allows us to investigate how we can achieve that all agents are super experts with the constraints of some protocols known from the literature, such as the protocol $\CMO$ wherein each pair of agents can only call each other once (either $ab$ or $ba$ is allowed, but not both)~\cite{hvdetal.dynamicgossip:2019}. 

For example, consider again the sequence $ab;ac;ad$ after which agents $a$ and $d$ are experts. 
Agent $a$ may no longer be involved in any subsequent call according to $\CMO$. It is therefore impossible for her to get to know that everyone is an expert.  
So, common knowledge of a protocol comes with additional constraints. 
It may also come with additional advantages: in this case we can sometimes achieve common knowledge of termination under synchronous conditions, i.e., if all agents know how many calls have been made, even if they were not involved themselves in all those calls. 
We will report some such cases, in particular for $\CMO$: for example, after an extension of $ab;ac;ad$ with three more calls, all agents including $a$ are super experts. Unfortunately, if we also allow missed calls this may no longer be the case, namely when an agent who already is a super expert \emph{must} call another agent in order for all agents to become super experts. Such an extra complication can be overcome if agents have a notion of time, and if we allow a so-called \Skip\ action that merely stands for a tick of the clock. We will carefully distinguish all such modelling aspects.

\paragraph*{GoMoChe}
To find out what agents know, we need to consider all call sequences they consider possible. Such reasoning about call sequences is a non-trivial exercise. To automatically find and verify such protocol executions we used GoMoChe, the model checker for gossip protocols available at \url{https://github.com/m4lvin/GoMoChe}. Assuming synchrony, we only need to reason about finite sets of call sequences to verify knowledge. But assuming asynchrony we need to reason about infinite sets of call sequences of arbitrary finite length, which cannot be done with a model checker. However, it is often sufficient to verify ignorance, i.e., lack of knowledge, namely by producing two `witness' call sequences with opposite properties. Such witnesses can already be found for call sequences of `small' length, by reasoning about finite sets of call sequences of a certain maximal length. We also used the model checker GoMoChe for that.

It should be possible in principle to have a model checker for knowledge in the asynchronous setting as well, namely using the notion of redundant call as in~\cite{AptW18,logicofgossiping:2020}, that bounds the maximal length of a call sequence without redundant (non-informative) calls, and that therefore also makes the sets of indistinguishable call sequences (and the length of individual sequences) finite again.

\paragraph*{Applications}
We hope our results may be useful for applications involving gossip. First, because we show that by merely exchanging secrets also stronger epistemic goals may be met. In other words, the cost of messaging remains unchanged so that at the price of longer call sequences we obtain the benefit of additional epistemic goals. We thus add to the standard gossip setting the possibility of acknowledgements (that all secrets have been disseminated as intended). Second, the more involved semantics where protocols are common knowledge might help to formalize as diverse settings as: councils or groups without a vertical hierarchy informing each other of the latest developments (when it is common knowledge that all calls are made only once, and in the presence of a time-out for rounds of calls). Third, the even more involved semantics where engaged agents do not answer calls: this is a widely used phenomenon knows as \emph{missed calls} wherein exactly such signalling takes place~\cite{Donner2007:Beeping}.

\paragraph*{Outline}
Section~\ref{section:def-part} presents a logical language and semantics for gossip protocols with the epistemic goal that all agents know that all agents know all secrets. A protocol is super-successful if all executions terminate satisfying this condition. We also recall four gossip protocols from the literature: $\ANY$, $\PIG$, $\CMO$, and $\LNS$. We obtain various results for the protocols $\ANY$ and $\PIG$, mainly that they are super-successful (both for the synchronous and asynchronous versions) in a sense adequate for protocols permitting arbitrarily large call sequences. 
Section~\ref{section:ComKnowP} refines the logic with common knowledge of gossip protocols. If a protocol $\prot$ is common knowledge we call it `known $\prot$.' We then show that synchronous known $\CMO$ is super-successful. 
Section~\ref{section:engaged} adds the feature that super experts do not make calls and do not answer calls. We then show that, if this is also known, super-successful protocol executions can be shorter. However, under these conditions $\CMO$ is no longer super-successful. 
We conclude in Section~\ref{section:conclusion} with an overview of our results and ideas for future research.

In three appendices we include further details.
Appendix~A contains the proof that protocol-dependent knowledge is well-defined.
Appendix~B adds \Skip\ calls to protocol-permitted sequences and applies this to known $\CMO$ with engaged agents.
Appendix~C provides scripts of the model checker GoMoChe that verify claims in this article.

\section{Gossip protocols for super experts}\label{section:def-part}

\subsection{Syntax and semantics}

Suppose a finite set of agents $A = \{a,b,c,\dots\}$ is given. 
We assume that two agents can always call each other, i.e., a complete network connects all the agents. 
Let $S \subseteq A^2$ be a binary relation such that we read $S_x y$ (for $(x,y) \in S$) as ``agent $x$ knows the secret of agent $y$'', and where $S_x$ stands for $\{ y \in A \mid S_x y\}$.
For the identity relation $\{(x,x) \mid x \in A\}$ we write $I$. 

The agents communicate with each other through telephone calls. 
During a call between two agents $x$ and $y$, they exchange all the secrets that they knew before the call. 
So if a call takes place the binary relation $S$ may grow. 

A \emph{call} or telephone call is a pair $(x,y)$ of agents $x,y \in A$ for which we write $xy$. Agent $x$ is the \emph{caller} and agent $y$ is the \emph{callee}. Given call $xy$, call $yx$ is the \emph{dual call}. An agent $x$ is \emph{involved} in a call $yz$ iff $y=x$ or $z=x$. A \emph{call sequence} is defined by induction: the empty sequence $\epsilon$ is a call sequence. 
If $\sigma$ is a call sequence and $xy$ is a call, then $\sigma;xy$ is a call sequence. 
Let $S$ be the secret relation between agents and $\sigma$ a call sequence. 
The result of applying $\sigma$ to $S$ is defined recursively by
\[ S^\epsilon = S; \text{ and } S^{\sigma;xy} = S^\sigma \cup (\{(x,y),(y,x)\} \circ S^\sigma) \]
where $\circ$ is relational composition: $R \circ R' := \{ (x,y) \mid \is z \ (x,z) \in R_1 \& (z,y) \in R_2\}$. 

We write $|\sigma|$ to denote the length of a call sequence, $\sigma[i]$ for the $i$th call of the sequence, $\sigma|i$ for the first $i$ calls of the sequence, and $\sigma_x$ for the subsequence of $\sigma$ that only contains calls involving $x$. If $\sigma = \rho;\tau$, then $\rho$ is a \emph{prefix} of $\sigma$, denoted as $\rho\sqsubseteq\sigma$, and $\tau$ is the \emph{complement} of $\rho$ in $\sigma$, where $\tau$ is also denoted $\sigma\setminus\rho$.

For a given set of agents $A$, a \emph{gossip state} is a pair $(S,\sigma)$, where $S$ is a secret relation and $\sigma$ a call sequence. A gossip state is \emph{initial} if $S=I$ and $\sigma=\epsilon$. In this contribution we only consider gossip states of the form $(I,\sigma)$, in which case we omit $I$. Hence $\epsilon$ stands for the initial state $(I,\epsilon)$, and $ab;cd$ stands for $(I,ab;cd)$, etcetera.

\begin{definition}[Language]\label{def.lang}
For a given finite set of agents $A$ the language $\lang$ of protocol conditions is given by the following $BNF$:
\[ \begin{array}{lcl}
\phi & := & \top \mid S_a b \mid Cab \mid \neg \phi \mid (\phi \wedge \phi) \mid K_a \phi \mid [\pi]\phi \\
\pi & := & ?\phi \mid ab \mid (\pi ; \pi) \mid (\pi \union \pi) \mid \pi^*
\end{array} \]
where $a,b$ range over $A$.
We have the usual abbreviations for implication, disjunction and for dual modalities, and often omit parentheses. 
\end{definition}
The $\phi$ are called \emph{formulas} and the $\pi$ are called \emph{programs}. The atomic formula $S_a b$ reads as `agent $a$ has the secret of $b$'. The atomic formula $Cab$ means that agent $a$ has called agent $b$ (in the past). The formula $K_a \phi$ reads `agent $a$ knows that $\phi$ is true'. By abbreviation we further define $\M_a \phi := \neg K_a \neg \phi$, for `agent $a$ considers it possible that $\phi$.' We also define the abbreviation $E \phi := \Et_{a \in A} K_a \phi$ and read it as `everyone knows that $\phi$' ($E \phi$ is also known as \emph{shared knowledge} or \emph{mutual knowledge} of $\phi$). Expression $[\pi]\phi$ reads as `after executing the program $\pi$, $\phi$ is true'. The basic program $?\phi$ denotes the `test on $\phi$' and the basic program $ab$ denotes, as expected, the call $ab$. The composite programs $\pi ; \pi$, $\pi \union \pi$ and $\pi^*$ represent, respectively, sequential execution, non-deterministic choice, and arbitrary iteration. Program iteration is defined as: $\pi^0 := {?\top}$, and for $n \geq 0$, $\pi^{n+1} := \pi^n;\pi$. 

Agent $a$ is an \emph{expert} if she knows all the secrets, formally $\bigwedge_{b \in A} S_a b$, abbreviated as $\Exp_a$. 
\emph{Everyone is an expert} is represented by the formula $\Exp_A := \bigwedge_{a \in A} \bigwedge_{b \in A} S_a b$. 
Agent $a$ is a \emph{super expert} if she knows that everyone is an expert, formally $K_a \Exp_A$. 

\begin{definition}[Protocol]\label{def.prot}
A protocol $\prot$ is a program defined by
 \[ \prot := {(?\neg E\!\Exp_A; \Union_{a\neq b \in A} (?\prot_{ab};ab))}^* ; ? E\!\Exp_A \]
where $\prot_{ab} \in \lang$ is the \emph{protocol condition} for call $ab$ of protocol $\prot$. 
\end{definition}
The formula $E\!\Exp_A$ is called the \emph{epistemic goal} or the \emph{termination condition} of the protocol.  
The non-deterministic choice in the definition of gossip protocol is common for \emph{distributed} protocols without central schedulers~\cite{apt-tark,AptW18,DitmarschGKP19}.
However, the difference from the usual definitions of a gossip protocol, e.g.~\cite{DitmarschGKP19}, is that we replace the goal $\Exp_A$ with $E\!\Exp_A$. In other words, instead of ``while not everyone is an expert, select two agents to make a call'' we have ``while not everyone is a super expert, select two agents to make a call''. 

To define formally what the agents participating in a gossip protocol \emph{know}, we use the following epistemic accessibility relations.
We recall that $I^\sigma_b$ is the set $\{ a \in A \mid (b,a) \in I^\sigma\}$, where $I$ is the secret relation that is the identity and where $I^\sigma$ is the result of applying call sequence $\sigma$ to relation $I$. Hence $a \in I^\sigma_b$ means that after the call sequence $\sigma$ agent $a$ knows the secret of $b$. Moreover, $I^\sigma_b = I^\tau_b$ means that after the two call sequences $\sigma$ and $\tau$ agent $b$ knows the same secrets. 

\begin{definition}[Epistemic relation]\label{def.synch}
Let $a \in A$. The \emph{synchronous epistemic relation} $\synch_a$ is the smallest equivalence relation between call sequences such that:
\begin{itemize}
\item $\epsilon \synch_a \epsilon$
\item if $\sigma \synch_a \tau$ and 
 $a \notin \{b,c,d,e\}$,
 then
 $\sigma;bc \synch_a \tau;de$
\item if $\sigma \synch_a \tau$ and 
 $I^\sigma_b = I^\tau_b$,
 then
 $\sigma;ab \synch_a \tau;ab$
\item if $\sigma \synch_a \tau$ and 
 $I^\sigma_b = I^\tau_b$,
 then
 $\sigma;ba \synch_a \tau;ba$
\end{itemize}
The \emph{asynchronous epistemic relation} $\asynch_a$ between call sequences is defined the same as the relation $\synch_a$ except that the second clause is replaced by
\begin{itemize}
\item if $\sigma \asynch_a \tau$,
 $a \notin \{b,c\}$,
 then
 $\sigma;bc \asynch_a \tau$.
\end{itemize}
\end{definition}
Informally, the synchronous accessibility relation encodes that agents not involved in a call are still aware that a call has taken place, as considered in~\cite{apt-tark,attamahetal.ecai:2014}. 
This implies that all agents know how many calls have taken place, i.e., there is a global clock. 
The asynchronous accessibility relation does not make any such assumption. 
Then, agents are only aware of the calls in which they are involved. 
Any information on other calls has to be deduced from the secrets they obtain from their calling partners. 

Both epistemic relations assume that the callers not only learn the union of the sets of secrets they each held before the call, but they also learn what set of secrets the other agent held before the call. As discussed in the introductory survey, this is known as the ``inspect-then-merge'' form of observation~\cite{AptGH18}.

Note that for any agent $a$, ${\synch_a} \subseteq {\asynch_a}$. This is fairly obvious, because for any call sequences $\sigma$ and $\tau$ and $a \notin \{ b,c,d,e \}$: $\sigma \asynch_a \tau$ implies $\sigma;bc \asynch_a \tau$, which implies $\sigma;bc \asynch_a \tau;de$. The latter copies the clause $\sigma;bc \synch_a \tau;de$ for the synchronous case.

\begin{definition}[Semantics]\label{Sema:Def}
Let call sequence $\sigma$ and formula $\phi\in\lang$ be given. 
We define $\sigma\models\phi$ by induction on the structure of $\phi$. 
\[ \begin{array}{lcl}
\sigma \models \top             & \text{iff} & true \\
\sigma \models S_a b            & \text{iff} & b \in I^\sigma_a \\
\sigma \models Cab              & \text{iff} & ab \in \sigma                                                                         \\
\sigma \models \neg \phi        & \text{iff} & \sigma\not \models \phi                                                               \\
\sigma \models \phi \wedge \psi & \text{iff} & \sigma \models \phi \ \text{and}\ \sigma \models \psi                                 \\
\sigma \models K_a \phi   & \text{iff} & \tau \models \phi \ \text{for all}\ \tau \ \text{such that}\ \sigma\synch_a\tau \\
\sigma \models [\pi] \phi       & \text{iff} & \tau \models \phi \ \text{for all}\ \tau \ \text{such that}\ \sigma\I{\pi}\tau
\end{array} \]
where
\[ \begin{array}{lcl}
\sigma\I{?\phi}\tau         & \text{iff} & \sigma \models \phi \ \text{and}\ \tau = \sigma                                            \\
\sigma\I{ab}\tau            & \text{iff} & \tau = \sigma;ab                                                                           \\
\sigma\I{\pi;\pi'}\tau      & \text{iff} & \text{there is a sequence}\ \rho \ \text{such that}\ \sigma\I{\pi}\rho \ \text{and}\ \rho\I{\pi'}\tau \\
\sigma\I{\pi\union\pi'}\tau & \text{iff} & \sigma\I{\pi}\tau \ \text{or}\ \sigma\I{\pi'}\tau                                          \\
\sigma\I{\pi^*}\tau         & \text{iff} & \text{there is}\ n \in \Naturals \ \text{such that}\ \sigma\I{\pi^n}\tau                   \\
\end{array} \]
The inductive clause for $K_a\phi$ above is for the synchronous setting. 
For the asynchronous setting we replace $\sigma\synch_a\tau$ by $\sigma\asynch_a\tau$ in that clause. 
For simplicity we do not use a separate symbol for the asynchronous semantics --- it will always be clear from the context what `$\models$' stands for. 
A formula $\phi$ is \emph{valid}, notation $\models\phi$,  iff for all call sequences $\sigma$ we have $\sigma\models\phi$. 
\end{definition}

We assume that all our protocols are \emph{symmetric}, which means that for all $a \neq b \in A$ and $c \neq d \in A$, simultaneously replacing $a$ by $c$ and $b$ by $d$ in the protocol condition $\prot_{ab}$ yields $\prot_{cd}$. 
Intuitively, a symmetric protocol gives the same instructions and does not assign any special roles to individual agents.
Moreover, we only consider protocols that are \emph{epistemic}, which means that $\prot_{ab} \to K_a \prot_{ab}$ is valid.
This means that agents always know which calls they are allowed to make (see~\cite[page~170]{DitmarschGKP19}).

If in call $ab$ agent $a$ or $b$ becomes an expert, then the other agent simultaneously becomes an expert, whereas if in a call $ab$ agent $a$ or agent $b$ becomes a super expert, then the other agent need not also become a super expert.

We continue with terminology on protocol termination. 
If $\sigma \models \prot_{ab}$ we say that call $ab$ is \emph{$\prot$-permitted} after $\sigma$. 
A $\prot$-permitted call sequence is a call sequence consisting of $\prot$-permitted calls. 

A call sequence $\sigma$ is \emph{$\prot$-maximal} if it is $\prot$-permitted and for any call $ab$, $\sigma;ab$ is not $\prot$-permitted. 

For subsequent definitions we also need to consider infinite call sequences. 
We denote an infinite call sequence as $\sigma_\infty$.  An infinite call sequence $\sigma_\infty$ is $\prot$-permitted if for any $i \in \Naturals$ prefix $\sigma_\infty|i$ is $\prot$-permitted.

A $\prot$-permitted infinite call sequence $\sigma_\infty$ is \emph{fair} if: for all $x \neq y \in A$,
  if for all $i$ there is $j>i$ such that $xy$ is $\prot$-permitted in $\sigma_\infty|j$,
  then for all $i$ there is $j>i$ such that $\sigma_\infty[j]=xy$.
  Fairness of an infinite call sequence means that all permitted calls are made infinitely often.  A $\prot$-maximal call sequence $\sigma$ is also called \emph{fair}. The idea of fairness is that all calls get a sporting chance to contribute to the dissemination of knowledge before it is too late and their execution is no longer permitted: that is why not any finite call sequence is fair but only a maximal one. Whereas with an infinite fair call sequence this moment is never reached, and calls remain permitted. Various kinds of fairness are discussed in~\cite{AptW18,LiveseyW22}.

A call sequence $\sigma$ is \emph{successful} if $\sigma\models\Exp_A$, and $\sigma$ is \emph{super-successful} if $\sigma\models E\!\Exp_A$. An infinite call sequence is (super-)successful if has a prefix that is (super-)successful. A protocol $\prot$ is \emph{(super-)successful} if all fair $\prot$-permitted finite and infinite call sequences are (super-)successful.

\paragraph*{Variants and terminology} We discuss a number of variants of gossip protocols and thus need to make clear for each section or result which specific variation it concerns. To distinguish different variants we use the following terminology. Protocols $\prot$ are considered for synchronous and for asynchronous conditions. Results for `$\prot$' where we do not mention synchrony or asynchrony explicitly, hold for both. Otherwise, we will speak of `synchronous $\prot$' or `asynchronous $\prot$'. In subsequent sections we will introduce gossip protocols that are common knowledge between agents and we will call this `known $\prot$'. Another variation will be a semantics where super experts do not make or answer calls and we will call this `known $\prot$ with engaged agents'. This means that a result for `$\prot$' may not be a result for `known $\prot$' or for `known $\prot$ with engaged agents', and vice versa. Protocol `$\prot$' as such, without further qualification, therefore means `without common knowledge of the protocol between agents' and `without engaged agents'.

\subsection{Gossip protocols ANY, LNS, CMO and PIG}
Four gossip protocols feature in this contribution. The protocol conditions are for any $a,b \in A$ with $a \neq b$.
\begin{itemize}
\item $\ANY$ with protocol condition $\ANY_{ab} := \top$; 
\item $\LNS$ with protocol condition $\LNS_{ab} := \neg S_a b$.
\item $\CMO$ with protocol condition $\CMO_{ab} := \neg Cab \et \neg Cba$;
\item $\PIG$ with protocol condition $\PIG_{ab} := \hat{K}_a \Vel_{c \in A} ((S_a c \et \neg S_b c) \vel (\neg S_a c \et S_b c))$;
\end{itemize}

Protocols with these protocol conditions have been investigated in the literature but with the weaker termination condition that all agents are experts ($\Exp_A$ instead of $E\!\Exp_A$). As these protocols have the same protocol condition, they come with the same protocol-permitted call sequences and we will therefore use the same name below.

The acronym $\ANY$ stands for \emph{make ANY call} and is the standard uninformed protocol in the gossip literature~\cite{kermarrecetal:2007}. There are infinite $\ANY$-permitted sequences, such as repeating the same call forever.

The acronym $\LNS$ stands for \emph{Learn New Secrets}. A call $ab$ is $\LNS$-permitted if agent $a$ does not know the secret of agent $b$~\cite{attamahetal.ecai:2014,DitmarschEPRS17,hvdetal.dynamicgossip:2019}. This protocol is traditionally known as NOHO, for \emph{No One Hears Own}~\cite{hedetniemietal:1988}. All $\LNS$-permitted sequences are finite, as each agent will call any other agent at most once, and will therefore makes at most $|A|-1$ calls.

The acronym $\CMO$ stands for \emph{Call Me Once}. You are allowed to call an agent if you have not yet been involved in a call with that agent. This protocol was introduced in~\cite{hvdetal.dynamicgossip:2019} and is reminiscent of~\cite{doerr}. 
As any two out of $n$ agents are only allowed to call each other once, the maximum number of calls in $\CMO$ is $\binom{n}{2} = \frac{n(n-1)}{2}$.

The acronym $\PIG$ stands for \emph{Possible Information Growth}. Intuitively, the call $ab$ is permitted if $a$ considers it possible that: $a$ will learn a secret $c$ that $b$ knows but not $a$, or that: $b$ will learn a secret $c$ that $a$ knows but not $b$. It has been investigated in~\cite{attamahetal.ecai:2014,DitmarschEPRS17}. Protocol $\PIG$ also permits infinite call sequences, although fewer then $\ANY$~\cite{DitmarschEPRS17}, as will also be discussed later.

Already with a merely strengthened epistemic goal we can obtain novel results for gossip protocols. The protocols $\ANY$ and $\PIG$ are super-successful, whereas $\LNS$ and $\CMO$ are not super-successful. First consider $\LNS$, in which case we can be short, by way of an example.

\begin{example}
Consider four agents and the call sequence $\sigma = ab;cd;ac;bd$ from the introduction and Table~\ref{tab:ab;cd;ac;bd}.
  This sequence is $\LNS$-maximal and successful.
  However, after $\sigma$ agent $a$ also considers the sequence $\tau = ab;cd;ac;bc$  possible, where the last call is different. After $\tau$, agent $d$ is not an expert, so that $\tau$ is not successful, and therefore after $\sigma$ agent $a$ does not know that everybody is expert, thence $\sigma$ is not super-successful. Similar examples can be found for $n > 4$ agents, hence $\LNS$ is not super-successful.
\end{example}

In the case of $\LNS$, the protocol condition is too restrictive to ever obtain super-successful termination. It is therefore not worth investigating further. However, it contrasts well with the protocol condition for $\PIG$ that contains similar constituents in an epistemic way, which makes it worth investigating. From here on, we no longer consider $\LNS$.

The protocol $\CMO$ is also not super-successful, but in this case we can recoup super-successful termination with the more involved semantics of Section~\ref{section:ComKnowP} (and only there we will show that $\CMO$ is not super-successful without), which can then be lost again with the even more involved semantics of Section~\ref{section:engaged}. We therefore employ it as a showcase illustrating the semantics of these sections.

\subsection{Results for the protocol ANY}\label{section:resultsany}

We first show that $\ANY$ is super-successful. This is followed by some examples of that for three and for four agents that also demonstrate the difference between asynchrony and synchrony.

\begin{theorem}\label{prop.anysuper}
  $\ANY$ is super-successful.
\end{theorem}
\begin{proof} 
We first observe that there are no maximal $\ANY$-permitted call sequences.
If $\sigma$ is finite then any call $ab$ is $\ANY$-permitted after $\sigma$ which contradicts maximality.

Let $\sigma_\infty$ be a fair $\ANY$-permitted call sequence. In other words, this is simply a call sequence in which all calls occur infinitely often.
  Hence there is a finite call sequence $\tau \sqsubset \sigma_\infty$ in which each call occurs.
  After $\tau$ everyone is an expert.

By fairness each call also occurs in the complement $\sigma_\infty\setminus\tau$.
  Hence there is a finite call sequence $\rho$ in which each call occurs and such that $\tau;\rho \sqsubset \sigma_\infty$. Because after $\tau$ everyone is an expert, for each call in  $\rho$ the agents involved know that they are both experts. Hence everyone is a super expert after $\rho$. Therefore $\sigma_\infty$ is super-successful, and because it was arbitrary we have shown that $\ANY$ is super-successful.
  
We did not use any assumption about synchrony or asynchrony. Therefore the result holds for both.
\end{proof}

\begin{example}\label{ex.threea}
Let $A=\{a,b,c\}$, and let the protocol be asynchronous $\ANY$. 
We show that after call sequence $ab;ac;ab;cb$ it holds that $E\!\Exp_A$. 
\begin{itemize}
\item After the prefix $ab;ac$, agents $a$ and $c$ are experts. 
\item After the prefix $ab;ac;ab$, agents $a$ and $b$ are super experts. 

Agent $a$ already knew that $c$ is an expert and in call $ab$ also learns that $b$ now is an expert. 
 Therefore, she is a super expert: $ab;ac;ab \models K_a \Exp_A$. 
 
In the third call, $ab$, agent $b$ learns that $a$ is an expert. Because in the first call $ab$ agent $a$ did not know the secret of $c$ yet, but now gives it to $b$, agent $b$ can infer that the call $ac$ must have taken place between the two $ab$ calls. 
 As in that call $ac$ agent $c$ became an expert, agent $b$ also knows that agent $c$ is an expert. 
 Therefore also agent $b$ is a super expert. 
\item Now consider the entire sequence $ab;ac;ab;cb$. In final call $cb$, agent $c$ becomes a super expert. After the second call, $ac$, agent $a$ is an expert, hence $c$ knows this. After the last call $cb$ agent $b$ is an expert, hence $c$ also knows this.
Therefore agent $c$ knows that all agents are experts. 
\end{itemize}
The sequence $ab;ac;ab;cb$ is minimal: for three agents using asynchronous ANY there is no sequence of less than four calls that is super-successful --- see Appendix C for a GoMoChe query to check this.
\end{example}

\begin{example}
Now assume synchrony. Consider the prefix $ab;ac;ab$ of the call sequence $ab;ac;ab;cb$ of Example~\ref{ex.threea}. This prefix is already synchronously super-successful. Agent $c$ is not involved in the third call, and this is common knowledge to all agents. All three agents only consider $ab;ac;ab$ possible.
\end{example}

\begin{example}\label{ex.fourany}
Let now $A=\{a,b,c,d\}$, and let the protocol be asynchronous $\ANY$.  
A super-successfully terminating sequence $ab;cd;ac;bd;ab;ad;bc;cd$ consisting of eight calls was already given in the introductory Section~\ref{section:introduction}. 

However, we can reach $E\!\Exp_A$ in only seven calls, namely with sequence:
\[ ab; cd; ac; ad; bc; ba; bd \]
What agents learn in these calls is shown in Table~\ref{tab:sevenCalls} (generated using the model checker). Let us sketch the justification of these results. 

\begin{table}
  \centering
  \begin{tabular}{l|l@{}l@{}l@{}l@{ }l@{}l@{}l@{}l|l@{}l@{}l@{}l@{ }l@{}l@{}l@{}l|l@{}l@{}l@{}l@{ }l@{}l@{}l@{}l|l@{}l@{}l@{}l@{ }l@{}l@{}l@{}l|l}
    \           & a& & & & & & &  &  &b& & & & & &  &  & &c& & & & &  &  & & &d& & & &  & initial state         \\
    $\call{ab}$ & a&b& & & & & &  & a&b& & & & & &  &  & &c& & & & &  &  & & &d& & & &  &                          \\
    $\call{cd}$ & a&b& & & & & &  & a&b& & & & & &  &  & &c&d& & & &  &  & &c&d& & & &  &                          \\
    $\call{ac}$ & a&b&c&d&A& &C&  & a&b& & & & & &  & a&b&c&d&A& &C&  &  & &c&d& & & &  &                          \\
    $\call{ad}$ & a&b&c&d&A& &C&D & a&b& & & & & &  & a&b&c&d&A& &C&  & a&b&c&d&A& &C&D &                          \\
    $\call{bc}$ & a&b&c&d&A& &C&D & a&b&c&d& &B&C&  & a&b&c&d&A&B&C&D & a&b&c&d&A& &C&D & $K_c \Exp_A$      \\
    $\call{ba}$ & a&b&c&d&A&B&C&D & a&b&c&d&A&B&C&  & a&b&c&d&A&B&C&D & a&b&c&d&A& &C&D &                          \\
    $\call{bd}$ & a&b&c&d&A&B&C&D & a&b&c&d&A&B&C&D & a&b&c&d&A&B&C&D & a&b&c&d&A&B&C&D & $E\!\Exp_A$ \\
  \end{tabular}
  \caption{Results of $ab;cd;ac;ad;bc;ba;bd$ with synchronous $\ANY$. We describe what an agent knows: a lower case $y$ in the column of $x$ means $S_x y$; an upper case $Y$ means $K_x \Exp_y$. Therefore, ``abcd'' denotes an expert and ``ABCD'' denotes a super expert.}\label{tab:sevenCalls}
\end{table}

After prefix $ab;cd;ac;ad$ we have three experts $a$, $c$ and $d$.
In the fifth call $bc$ agent $b$ becomes an expert, and as usual $b$ and $c$ learn about each other that they are experts.
In addition, and somewhat surprisingly, $c$ also learns in that call that $d$ is an expert. This is due to synchrony and can be checked as follows: $c$ knows that between the third call $ac$ and the fifth call $bc$ there must have been a call which must have between between $a$ and $d$ or between $a$ and $b$.
But in the fifth call $bc$ agent $b$ only knows the secrets of $a$ and $b$, hence this fourth call did not involve $b$. Therefore, it must have involved $d$, which implies that $d$ is an expert. (See Table~\ref{tab:sevenCalls}). 

Note that agent $c$ only became a super expert in call $bc$ because of synchrony, and that $c$ is not involved in calls after that, and therefore asynchronously considers it possible that $bc$ was the last call. Therefore, this seven-call sequence is not super-successful asynchronously. By exhaustive search in the model checker GoMoChe  we confirmed that other call sequences of at most seven calls are also not super-successful (Appendix C).
\end{example}

The above examples demonstrated that:
\begin{observation}\label{prop:syncAnyShorterSupSuccThanAsyncAny}
For three and for four agents, synchronous $\ANY$ permits shorter super-successful sequences than asynchronous $\ANY$.
\end{observation}
We have not shown that for any number of agents synchronous $\ANY$ permits shorter super-successful sequences than asynchronous $\ANY$.
We already do not know whether this is the case for five agents.

We continue with observations relating expert to super expert, and different super experts.
\begin{observation}\label{obs.synch}
  Assuming synchrony, an agent can become an expert and a super expert in the same call.
\end{observation}

\begin{example}\label{ex.6Calls}
Consider $4$ agents, synchronous $\ANY$ and $ab;ac;cd;ab;bc;ab$. In the final call, agent $a$ becomes an expert and a super expert. See Table~\ref{tab:6Calls}. This sequence was found and the table was generated using the model checker GoMoChe (see Appendix C). It is easy to explain the entries in Table~\ref{tab:6Calls} and to see why $a$ is a super expert. 

\begin{table}[H]
  \centering
  \begin{tabular}{l|l@{}l@{}l@{}l@{ }l@{}l@{}l@{}l|l@{}l@{}l@{}l@{ }l@{}l@{}l@{}l|l@{}l@{}l@{}l@{ }l@{}l@{}l@{}l|l@{}l@{}l@{}l@{ }l@{}l@{}l@{}l|l}
    \           & a& & & & & & &  &  &b& & & & & &  &  & &c& & & & &  &  & & &d& & & &  & initial state          \\
    $\call{ab}$ & a&b& & & & & &  & a&b& & & & & &  &  & &c& & & & &  &  & & &d& & & &  &                          \\
    $\call{ac}$ & a&b&c& & & & &  & a&b& & & & & &  &  a&b&c& & & & &  &  & & &d& & & &  &                          \\
    $\call{cd}$ & a&b&c& & & & &  & a&b& & & & & &  &  a&b&c&d& & &C&D&  a&b&c&d& & &C&D&                          \\
    $\call{ab}$ & a&b&c& & & &C&D& a&b&c& & & & &  &  a&b&c&d& & &C&D&  a&b&c&d& & &C&D&                          \\
    $\call{bc}$ & a&b&c& & & &C&D& a&b&c&d& &B&C&D&  a&b&c&d& &B&C&D&  a&b&c&d& & &C&D&                          \\
    $\call{ab}$ & a&b&c&d&A&B&C&D& a&b&c&d&A&B&C&D&  a&b&c&d& &B&C&D&  a&b&c&d& {\color{white}A}& {\color{white}B}&C&D&    {\small $a$ is expert and} \\
        \           & & & & & & & &  &  & & & & & & &  &  & & & & & & &  &  & & & & & & &  &   {\small super expert}                      \\
  \end{tabular}
  \caption{Results of $ab;ac;cd;ab;bc;ab$ with synchronous $\ANY$.}\label{tab:6Calls}
\end{table}

In the second call $ab$ that is the fourth call in the sequence, $a$ considers it possible that the third call, not involving her, was: $bc$, $bd$ or $cd$. As in the fourth call she learns that agent $b$ only knew $\{A,B\}$, she can therefore rule out that $b$ was involved in the third call. It must therefore have been $cd$. Also, $a$ knows that $c$ brought secrets $\{A,B,C\}$ into that call $cd$, that $c$ learnt in the second call, $ac$. Therefore $a$ learns that $c$ and $d$ became experts in the third call. But $a$ is not yet an expert herself after this second call $ab$.

After the third call $ab$ that is the sixth call in the sequence, $a$ is an expert, and $a$ knows that $b$ is also an expert. Therefore, she is now a super expert.
\end{example}

\begin{observation} Two agents can become super experts in the same call.
\end{observation}

\begin{example}
It is straightforward to see that an asychronous example is call sequence $ab;cd;ac;bd;ad;bc;ab$ from the above Example~\ref{ex.fourany}. In the final call, $a$ and $b$ become super experts. Now just remove the final call, and it is obvious that neither $a$ nor $b$ is a super expert: the remaining final two calls $ad;bc$ are made by disjoint pairs of agents that both consider it possible that the other call did not happen. 

A synchronous example is given by the last row of Table~\ref{tab:6Calls}, wherein $a$ and $b$ simultaneously become super experts by calling each other.
\end{example}

How many calls are needed for at least one agent to become a super expert?
\begin{proposition}\label{ex.threes}
For $n \geq 2$ agents, an agent can become a super expert in $2n-3$ calls.
\end{proposition}
\begin{proof}
  Let there be $n \geq 2$ agents. Let an agent call other agents in succession. These are $n-1$ calls. Let that agent call all other agents again in succession except the last one. These are $n-2$ calls. Then this agent is now a super expert. Altogether, these are  $(n-1)+(n-2)=2n-3$ calls. 
  
Notably, this sequence makes the caller a super expert,   no matter whether we assume  synchrony or asynchrony, because the caller is involved in all calls. The analysis does not involve epistemic reasoning sometimes allowing for shorter sequences with synchrony. Synchronously the caller only considers one call sequence, asynchronously many more, but the shortest one involving herself in all calls remains the same.
 \end{proof}
An example of Proposition~\ref{ex.threes} for $4$ agents is the ($2\cdot 4 -3=$) five-call sequence $ab;ac;ad;ab;ac$. It makes agent $a$ a super expert. The call $ad$ is not needed twice, because $d$ already became an expert in call $ad$.
   
We are uncertain whether the bound in Proposition~\ref{ex.threes} is hard. However, if there is a super expert, then all agents are experts, and merely one less call, $2n-4$, is the minimum number for all of $n \geq 4$ agents to become experts~\cite{tijdeman:1971}. For $n=2$, one call $ab$ is the minimum and for $n=3$ the call sequence $ab;ac;ab$ is the minimum. There are different ways for all agents to become experts in $2n-4$ calls and for $n \geq 4$ none known to us result in an agent also becoming a super expert in the final call. The number $2n-3$ figures in various results for gossip with the expert goal: the conjectured minimum for termination of an \emph{epistemic} gossip protocol is $2n-3$ (that is, with protocol conditions known by the calling agent, so that one cannot require that the first two calls are disjoint)~\cite{boukeLOFT,abs-1907-09097}; the minimum number of calls for various network topologies where agents cannot call all other agents but only some (their `neighbours') is $2n-3$, for example circles and binary trees~\cite{hedetniemietal:1988}.

We can use Example~\ref{ex.fourany} to show how many calls can make everyone a super expert.

\begin{proposition}\label{prop.minmin}
Given $n \geq 4$ agents, super-successful asynchronous $\ANY$ termination can be achieved in $n - 2 + \binom{n}{2}$ calls.
\end{proposition}
\begin{proof}
  Consider $n$ agents, select $4$ agents $a,b,c,d$ among these $n$ and $1$ agent $a$ among these $4$. First, let $a$ call all the agents except $b,c,d$. These are $(n-4)$ calls. Then, let $a,b,c,d$ execute the sequence $ab;cd;ac;bd$. These are $4$ calls. Note that in the final two calls $ac$ and $bd$ these four agents become experts. Apart from $ac$ and $bd$, we now let all remaining pairs of agents also call each other. There are $\binom{n}{2}$ pairs of agents (including $ac$ and $bd$). When after a call both agents are experts, they know this from one another. Therefore, after the $\binom{n}{2}$ calls, all agents know that all agents are experts: $E\Exp_A$ holds.
  Altogether these are $(n-4) + 4 -2 + \binom{n}{2} = n-2 + \binom{n}{2}$ calls. 
\end{proof}
In the proof of Proposition~\ref{prop.minmin}, the first call in which two agents become experts is call $n-1$. This is the minimum, as $n-1$ links are need to connect $n$ points in a graph. No agents can become experts in the first $n-2$ calls. In all subsequent $\binom{n}{2}$ calls, when calling each other, either agents $x$ and $y$ become expert, or they learn from each other that they already were experts. This suggests that the only way in which an agent asynchronously can get to know that another agent is an expert is by calling that agent. The next example shows that this is false.

\begin{example}\label{example.noprior}
Assume asynchrony. Consider four agents and the call sequence $\sigma=ac;ad;ac;bc;ac$. After the sequence $ac;ad;ac$ these three agents share their secrets. In call $ad$ agent $a$ learns that $d$ has not been involved in a call with $b$ and in the second call $ac$ agent $a$ learns that $c$ has not been involved in a call with $b$ after the first call $ac$. Therefore $a$ knows that whomever $b$ makes his first call with, he will become expert. In the third call $ac$ of $\sigma$ agent $a$ learns that $c$ knows the secret of $b$, so there should have been a call between $b$ and $c$ or between $b$ and $d$. (If between $b$ and $d$, that call could have taken place between call $ad$ and the second call $ac$, but not if between $b$ and $c$.) Either way, $b$ then would be an expert. \emph{So $a$ knows that $b$ is an expert.} However, there has been no prior call between $a$ and $b$ wherein they both became or already were experts.

The model checker GoMoChe confirms that no super-successful asynchronous seven-call sequence exists  (see Appendix C), and also that no super-successful asynchronous eight-call sequence exists extending $\sigma$. So, this prefix $\sigma=ac;ad;ac;bc;ac$ is not an efficient start in order to get super-successful termination.
\end{example}
Because of phenomena like in Example~\ref{example.noprior} we are uncertain whether the bound in Proposition~\ref{prop.minmin} is hard.

\subsection{Results for the protocol PIG}\label{section:resultspig}

The $\PIG$ protocol has infinite executions ($\PIG$-permitted sequences) for four or more agents~\cite{DitmarschEPRS17}. 

Infinite call sequence $ab;ab;ab;\dots$ is asynchronous $\PIG$-permitted. Call $ab$ is indistinguishable for agent $a$ from call sequence $ab;bc$, after which agent $b$ has learnt something new. Thus, after first call $ab$, the same call $ab$ is again $\PIG$-permitted. Similarly, $ab;ab \asynch_a ab;ab;bc$, thus $ab$ is again $\PIG$-permitted after $ab;ab$, and so on. However, this infinite call sequence is not synchronous $\PIG$-permitted (although it is $\ANY$-permitted).

Infinite call sequence $ab;cd;ab;cd;ab;cd;\ldots$ is synchronous $\PIG$-permitted, as after any even number of calls agent $a$ considers it possible that agent $b$ was involved in the previous call and would thus have learnt a new secret in that call. Therefore, each odd call can again be call $ab$. 

Unlike for $\ANY$, there are maximal $\PIG$-permitted call sequences: if in call sequence $\tau;ab$ agents $a$ and $b$ become expert in that final call $ab$, then call $ab$ is not $\PIG$-permitted in any extension of $\tau;ab$.

\begin{example}\label{ex.pig}
The super-successful call sequence $\sigma = ab;cd;ac;bd;ab;ad;cb;cd$ from Section~\ref{section:introduction} is not only $\ANY$-permitted but also $\PIG$-permitted. It is also $\PIG$-maximal, but obviously not $\ANY$-maximal, as no $\ANY$  call sequence is maximal. 

We can adapt $\sigma$ to get a super-successful $\ANY$-permitted sequence that is not $\PIG$-permitted: in $\sigma$, repeat penultimate call $cb$ before final call $cd$, i.e., with the additional call in bold, $ab;cd;ac;bd;ab;ad;cb;\pmb{cb};cd$.
\end{example}

To obtain termination results for the $\PIG$ protocol we need to analyze finite call sequences and infinite call sequences. We will first show a relation between the $\PIG$ protocol condition and the termination condition super-successful, and we will then show that $\PIG$ is super-successful.

\begin{lemma}\label{lemma:VelPigIffNotEExp}
$\Vel_{a,b\in A} \PIG_{ab} \eq \neg E\!\Exp_A$ is valid.
\end{lemma}
\begin{proof}
We recall that for any call $ab$, $\PIG_{ab} := \hat{K}_a \Vel_{c \in A} ((S_a c \et \neg S_b c) \vel (\neg S_a c \et S_b c))$.

Assume $\Vel_{a,b\in A} \PIG_{ab}$. 
If an agent $a$ considers it possible that there is a secret that is not known by another agent $b$ or by herself, then she considers it possible that that other agent or herself is not an expert: $\neg K_a \neg \neg \Exp_b \vel \neg K_a \neg \neg \Exp_a$. Either way, she then does not know that all agents are experts, $\neg K_a \Exp_A$, and therefore $\neg E\!\Exp_A$.

For the other direction, suppose $\lnot E\!\Exp_A$.
Then for some three agents $a,b,c$ we have $\hat{K}_a \lnot S_b c$.
We distinguish the case where $S_a c$ holds from the case where $\neg S_a c$ holds. If $S_a c$, then $K_a S_a c$. Thus also $\hat{K}_a (S_a c \land \lnot S_b C) $, which implies $\PIG_{ab}$. If $\lnot S_a c$, then we have $\PIG_{ac}$. In both cases we get $\Vel_{a,b\in A} \PIG_{ab}$.
\end{proof}

\begin{theorem}\label{Thm:pig}
$\PIG$ is super-successful. 
\end{theorem}
\begin{proof}
First, let $\sigma$ be a $\PIG$-maximal call sequence. This means that $\sigma \models \Et_{a,b\in A} \lnot \PIG_{ab}$. By Lemma~\ref{lemma:VelPigIffNotEExp} we thus have $\sigma \models E\!\Exp_A$, i.e. $\sigma$ is super-successful.

Next, let $\sigma_\infty$ be a $\PIG$-permitted infinite call sequence.
Towards a contradiction suppose $\sigma_\infty$ is fair.
    
Let $\tau \sqsubset \sigma_\infty$. As $\sigma_\infty$ is $\PIG$-permitted, there is an agent $a$ for which there is an agent $b$ such that $\tau \models \PIG_{ab}$, that is, $\tau\models \M_a \Vel_{c \in A} (S_a c \et \neg S_b c) \vel (\neg S_a c \et S_b c)$. 

First assume that there is a $c$ such that $\tau\models \M_a\neg S_a c$. Then $ac$ is  $\PIG$-permitted after $\tau$, so, as $\sigma_\infty$ is fair, there is a $\rho$ with $\tau \sqsubset \rho \sqsubset \sigma_\infty$ with $ac \in \rho\setminus\tau$. And then $\rho\models S_a c$. We can do this for all such $c$. Therefore, there is $\rho'$ with $\rho \sqsubset \rho' \sqsubset \sigma_\infty$ after which $a$ is an expert. 
 
At that stage it can still be that $a$ considers it possible that some other agent $b$ is not an expert: $\rho' \models \M_a \Vel_{c \in A} (S_a c \et \neg S_b c)$. In that case, $ab$ is still $\PIG$-permitted after $\rho'$. Therefore, as $\sigma_\infty$ is fair, there is a $\xi$ with $\rho' \sqsubset \xi \sqsubset \sigma_\infty$ such that $ab \in \xi\setminus\rho'$. As $a$ already was expert, $\xi \models K_a \Exp_b$. Again, we can do this for all such $b$. Therefore, there is $\xi'$ with $\xi \sqsubseteq \xi' \sqsubset \sigma_\infty$ after which $a$ is a super expert. 

We are almost there. At this stage it can still be that some agent $d$ other than $a$ considers it possible that not all other agents are expert. As $a$ is a super expert, $d$ must already be an expert. We then must have that $\xi'\models \M_d \neg S_e f$ for some agents $e,f$ which $\PIG$-permits call $de$ and we repeat the argument in the preceding paragraph so that we finally obtain a $\chi$ with $\xi'\sqsubset \chi\sqsubset\sigma_\infty$ after which $d$ is also a super expert. We do this for all such $d$, so that there is a $\chi'$ with $\chi\sqsubseteq \chi'\sqsubset\sigma_\infty$ after which $E\!\Exp_A$ holds.
    
By Lemma~\ref{lemma:VelPigIffNotEExp} we then get $\chi' \models \lnot \Vel_{a,b\in A} \PIG_{ab}$ which means that any remaining calls in $\sigma_\infty$ after $\chi'$ are not $\PIG$-permitted. This contradicts that $\sigma_\infty$ is $\PIG$-permitted.
  
  Hence there are no fair infinite $\PIG$-permitted call sequences.
\end{proof}

\section{Common knowledge of gossip protocols}\label{section:ComKnowP}

\subsection{Syntax and semantics --- known protocols}
We now enrich the framework by modelling common knowledge of protocols. This requires that we replace the knowledge modality by a knowledge modality depending on a given protocol, and that we replace the epistemic relations by more restricted relations incorporating common knowledge of the protocols (it is a restriction as this reduces the uncertainty about call sequences). The resulting semantic framework is more complex, because these definitions require mutual recursion both in the syntax and in the semantics. In the syntax, because what an agent knows now depends on a given protocol, whereas the protocol is defined with respect to a protocol condition, that could be a knowledge formula, that needs to be evaluated in the semantics. Similarly, in the semantics, the epistemic relation (that interprets a knowledge modality) depends on a given protocol, and thus on the interpretation of the protocol conditions: formulas, so we are back in the syntax. We adapt the framework presented in~\cite{DitmarschGKP19} to our needs.

\begin{definition}[Language and Protocol --- known protocols]\label{def.prot2}
In the BNF of the language $\lang$ we replace the inductive clause $K_a \phi$ by an inductive clause $K^\prot_a \phi$. For $\Et_{a\in A}K^\prot_a\phi$ we write $E^\prot\phi$. Then, a \emph{protocol} called ``\emph{known $\prot$}'' is a program defined by
 \[ \prot := {(?\neg E^\prot\!\Exp_A ;\Union_{a\neq b \in A} (?\prot_{ab};ab))}^* ; ? E^\prot\!\Exp_A \]
\end{definition}
Formula $K_a^\prot \phi$ means that agent $a$ knows $\phi$ given (common knowledge between all agents of) protocol $\prot$. So, $E^\prot\!\Exp_A$ means that everyone is a super expert given protocol $\prot$. We call $K_a^\prot \phi$ \emph{protocol-dependent knowledge} (of $\phi$).

We now define $\synch_a^\prot$ and $\asynch_a^\prot$, simultaneously with the satisfaction relation $\models$. The difference with the prior Definition~\ref{Sema:Def} of $\models$, is that we replace $K_a$ by $K^\prot_a$ everywhere and $\synch_a$ by $\synch^\prot_a$ everywhere, and similarly for $\asynch_a$. Only the knowledge clause of the semantics is therefore given. 

\begin{definition}[Epistemic relations and semantics --- known protocols]\label{def.synch2} \ \\
Let $a\in A$. The synchronous accessibility relation $\synch_a^\prot$ between call sequences is the smallest symmetric and transitive relation such that:
\begin{itemize}
\item $\epsilon \synch_a^\prot \epsilon$,
\item if $\sigma \synch_a^\prot \tau$,
 $a \notin \{b,c,d,e\}$,
 $\sigma \models \prot_{bc}$ and
 $\tau  \models \prot_{de}$
 then
 $\sigma;bc \synch_a^\prot \tau;de$
\item if $\sigma \synch_a^\prot \tau$,
 $I^\sigma_b = I^\tau_b$,
 $\sigma \models \prot_{ab}$ and 
 $\tau  \models \prot_{ab}$,
 then
 $\sigma;ab \synch_a^\prot \tau;ab$
\item if $\sigma \synch_a^\prot \tau$,
 $I^\sigma_b = I^\tau_b$,
 $\sigma \models \prot_{ba}$ and 
 $\tau  \models \prot_{ba}$,
 then
 $\sigma;ba \synch_a^\prot \tau;ba$
\end{itemize}
The asynchronous accessibility relation $\asynch_a^\prot$ between call sequences is the same as the relation $\synch_a^\prot$ except that the second clause is replaced by
\begin{itemize}
\item if $\sigma \asynch_a^\prot \tau$,
 $a \notin \{b,c\}$, and
 $\sigma \models \prot_{bc}$,
 then
 $\sigma;bc \asynch_a^\prot \tau$
\end{itemize}
Finally, in the inductive definition of $\models$ we replace the clause for $K_a \phi$ by: 
\[ \begin{array}{lcl}
\sigma \models K_a^\prot \phi   & \text{iff} & \tau \models \phi \ \text{for all}\ \tau \ \text{such that}\ \sigma\synch_a^\prot\tau
\end{array} \]
\end{definition}

On the set of $\prot$-permitted call sequences the relations $\synch^\prot_a$ and $\asynch^\prot_a$ are equivalence relations, but not on the set of all call sequences: below, we show that the {\bf T} axiom $K^\prot_a \phi \imp \phi$ fails so that the relation is not reflexive and therefore no equivalence relation; see also~\cite{DitmarschGKP19}.

For $K^\ANY_a\phi$ we write $K_a \phi$,
for $\synch_a^\ANY$ we write $\synch_a$ and for $\asynch_a^\ANY$ we write $\asynch_a$. This is not ambiguous, because if for all $a,b \in A$, $\prot_{ab} = \top$, we regain the syntax and semantics of the previous Section~\ref{section:def-part}.

In Definition~\ref{def.prot2} of the version of the language and the protocols assuming commonly known protocols, formula $K^\prot_a \phi$ contains as parameter a protocol $\prot$, and vice versa a protocol $\prot$ contains protocol conditions $\prot_{ab}$ that are formulas. This is well-defined, once we see $K^\prot_a\phi$ as $K_a(X,\phi)$ where $X$ is the list of formulas $\prot_{ab}$ for $a \neq b \in A$. In other words, we see $K^\prot_a \phi$ as a modality with not a single argument $\phi$, but with $|A|^2 - |A| + 1$ arguments.\footnote{Namely, the combinations $\binom{|A|}{2}$ of $2$ out of $|A|$ elements, times $2$ as for any different $a,b \in A$ formulas $\prot_{ab}$ and $\prot_{ba}$ count separately, plus $1$ for the formula $\phi$ bound by the $K^\prot_a$ modality.} For formal precision, in Appendix A we give the well-founded preorder demonstrating that the semantics is well-defined. As also discussed at length in~\cite{DitmarschGKP19}, this excludes self-referential protocols.

As already explicit in Definition~\ref{def.prot2}, we refer to the protocol $\prot$ with the syntax and semantics for common knowledge of protocols as \emph{known $\prot$}. The properties of protocol `known $\prot$' may be very different from those of protocol $\prot$, except that known $\ANY$ is the same as $\ANY$, as said above. To summarize the differences: the protocol $\prot$ and the protocol known $\prot$ have the same protocol conditions $\prot_{ab}$ for all calls $ab$ but have different termination condition, namely $E\!\Exp_A$ respectively $E^\prot\!\Exp_A$. Therefore, the set of $\prot$-permitted call sequences remains the same either way. Protocol-dependent knowledge $K^\prot_a\phi$ is merely a way to use the information already available in the set of $\prot$-permitted call sequences differently, namely by defining the relations $\sim_a^\prot$ and $\approx_a^\prot$. There is nothing against using different $K^\prot_a$ and $K^{\prot'}_a$ in the same formula. Implicitly, we have already done this: the protocol condition of $\PIG$ remains the same, namely $\hat{K}_a \Vel_{c \in A} ((S_a c \et \neg S_b c) \vel (\neg S_a c \et S_b c))$. Formally, this is now $\hat{K}_a^\ANY \Vel_{c \in A} ((S_a c \et \neg S_b c) \vel (\neg S_a c \et S_b c))$; whereas the termination condition of known $\PIG$ is $E^\PIG\!\Exp_A$.

We list some elementary properties of the semantics below, but refer to~\cite{DitmarschGKP19} for further discussion and proofs.
Here, $a,b \in A$, protocols $\prot$, $\prot'$, and $\phi \in \lang$ are all arbitrary. 
\begin{itemize}
\item $\models K^\prot_a \phi \imp K^\prot_a K^\prot_a \phi$, and $\models \neg K^\prot_a \phi \imp K^\prot_a \neg K^\prot_a \phi$. 
 Intuitively, $K^\prot_a$ has two of the standard properties of knowledge, namely positive and negative introspection. 
\item $\not\models K^\prot_a \phi \imp \phi$. Whenever $\sigma$ is not $P$-permitted, then $\sigma \models K^\prot_a\bot$. In other words, if you are in violation of the protocol, anything goes. However, whenever $\sigma$ is $\prot$-permitted, then $\sigma \models K^\prot_a \phi \imp \phi$.
\item $\models \prot_{ab} \imp \prot'_{ab}$ implies $\models K_a^{\prot'} \phi \imp K_a^\prot\phi$; as $K^\ANY_a \phi = K_a \phi$, for all $a,b \in A$, $\ANY_{ab}=\top$ and $\psi \imp \top$ is valid for all $\psi$, a corollary is that $\models K_a \phi \imp K_a^\prot\phi$.
\item $\models S_a b \eq K_a^\prot S_a b$ and $\models \neg S_a b \eq K_a^\prot \neg S_a b$. Whether $a$ knows the secret of $b$ can be determined from the call sequence and independently from the protocol.
\end{itemize}

\subsection{Results for the protocols PIG and CMO}\label{section:resultscmo}

As known $\ANY$ is the same as $\ANY$, and $\ANY$ is super-successful (Theorem~\ref{prop.anysuper}), known $\ANY$ is also super-successful. This is therefore not separately given as a result. The protocol known $\PIG$ is also still super-successful. This is non-trivial, as $K_a^\PIG \phi$ need not be equivalent to $K_a \phi$. The minor results in this section are called propositions instead of theorems. 
\begin{proposition}\label{thm:xxx}
Known $\PIG$ is super-successful.
\end{proposition}
\begin{proof}
The proof is a straightforward but interesting variation of the proof of Theorem~\ref{Thm:pig} that
$\PIG$ is super-successful. We use that from $\PIG_{ab} \imp \ANY_{ab}$ it follows that for all $\phi$, $K_a \phi \imp K^\PIG_a \phi$ (see the properties of protocol-dependent knowledge listed above), so in particular $K_a \Exp_A \imp K^\PIG_a \Exp_A$ and also $E\!\Exp_A \imp E^\PIG\!\Exp_A$.

In the proof, we still reason over $\PIG$-permitted call sequences, that is, sequences that satisfy for every next call $\PIG_{ab}$ which is $\M_a \Vel_{c \in A} (S_a c \et \neg S_b c) \vel (\neg S_a c \et S_b c)$, the standard $\M_a$, not $\M^\PIG_a$. The protocol condition is the same. But the termination condition is different: whereever $K_a \Exp_A$ is obtained in the proof, we now need $K_a^\PIG \Exp_A$, and whereever $E\!\Exp_A$ is obtained in the proof, we now need $E^\PIG\!\Exp_A$. But both follow from the observed implications above, so we are fine.

It is still the case, like in the proof of Theorem~\ref{Thm:pig}, that no fair infinite $\PIG$-permitted call sequences occur. However, it may be that $E^\PIG\!\Exp_A$ holds before $E\!\Exp_A$ in a maximal $\PIG$-permitted call sequence. (Although $K_a^\PIG\phi$ is not always equivalent to $K_a\phi$, we do not know whether $E^\PIG\!\Exp_A$ is equivalent to $E\!\Exp_A$.)
\end{proof}

For the protocol $\CMO$, whether the agents know that $\CMO$ is executed makes a big difference. It is the difference between being super-successful or not. We first list the negative results, followed by a positive result for synchronous known $\CMO$.

\begin{proposition}\label{thm:nonck-CMO}
Synchronous (not commonly known) $\CMO$ is not super-successful.
\end{proposition}
\begin{proof}
There are counterexamples whenever $|A| \geq 4$. 

Given $A=\{a_1,a_2,\ldots,a_n\}$, let $\rho$ be a maximal $\CMO$-permitted sequence between agents $\{a_1,a_2,\ldots,a_{n-1}\}$. 
From~\cite{hvdetal.dynamicgossip:2019} it follows that after $\rho$ all agents $a_1,a_2,\ldots,a_{n-1}$ know all their secrets. So they are all experts for the set $\{a_1,a_2,\ldots,a_{n-1}\}$ except that none knows the secret of $a_n$. 
Now define the call sequence $\sigma$ by having agent $a_n$ call all other agents after $\rho$:
\[ \sigma := \rho;a_n a_1;a_n a_2;\ldots; a_n a_{n-1} \]
We note that $\sigma$ is again a maximal $\CMO$ sequence, as $\binom{n-1}{2} + (n-1) = \binom{n}{2}$. 
After $\sigma$, all agents are experts, and agent $a_n$ is the only super expert. 
Let $i,j < n$ and $i \neq j$. 
Now consider the following call sequence $\tau$ where $a_n$ only calls $a_j$ (many times) and $a_i$ (once, at the same moment as in $\sigma$):
\[ \tau := \rho;\overbrace{a_n a_j;a_n a_j;\ldots a_n a_j}^{i-1 \text{ times }}; a_n a_i; \overbrace{a_n a_j; a_n a_j \ldots a_n a_j}^{ n-i-1 \text{ times }} \]
We then have that $\sigma \synch_{a_i} \tau$ and that $\tau \not\models \Exp_A$. 
Therefore, $\sigma \models \neg K_{a_i} \Exp_A$. 
As $\sigma$ is maximal and not super-successful, $\CMO$ is not super-successful. 
\end{proof}

\begin{proposition}\label{thm:async-CMO}
Asynchronous $\CMO$ is not super-successful.
\end{proposition}
\begin{proof}
There are counterexamples whenever $|A| \geq 4$. 

Consider again the call sequence $\rho$ and $\sigma$ from the proof of Proposition~\ref{thm:nonck-CMO}. 
The sequence $\rho;a_n a_i$ is $\CMO$-permitted, and $\sigma \asynch_{a_i} \rho;a_n a_i$. 
After $\rho;a_n a_i$, only agents $a_n$ and $a_i$ are experts but none of the remaining agents. 
Therefore, $\sigma \not\models K_{a_i} \Exp_A$, so $\sigma \not\models E\Exp_A$.
As $\sigma$ is maximal and not super-successful, $\CMO$ is not super-successful. 
\end{proof}

In fact, for the proof of Proposition~\ref{thm:async-CMO} it does not matter whether $\CMO$ is known, as we also have $\sigma \asynch^\CMO_{a_i} \rho;a_n a_i$. Hence we have the same result when $\CMO$ is known.

\begin{corollary}\label{cor:async-known-CMO}
Asynchronous known $\CMO$ is not super-successful.
\end{corollary}

\begin{example}\label{ex:cmo-sixCalls}
Consider the semantics without protocol knowledge. Let $A = \{a,b,c,d\}$ and consider the sequence $\sigma := ab;ac;bc;ad;db;dc$. 
This sequence is $\CMO$-permitted, $\CMO$-maximal, and satisfies $\Exp_A$. 

Observe that $\sigma \synch_b ab;ac;bc;ad;db;ad$, where in the call sequence on the right side we replaced the final call $dc$ in $\sigma$ by $ad$. 
This sequence is not $\CMO$-permitted, as call $ad$ occurs twice. 
After $ab;ac;bc;ad;db;ad$, agent $c$ does not know the secret of $d$, therefore $ab;ac;bc;ad;db;ad \not\models \Exp_A$. 
From that and $\sigma \synch_b ab;ac;bc;ad;db;ad$ then follows that $\sigma \not\models K_b \Exp_A$, and therefore $\sigma \not\models E\!\Exp_A$, so that $\sigma$ is not super-successful. 
\end{example}

\begin{example}
Consider again call sequence $\sigma$ from the previous Example~\ref{ex:cmo-sixCalls}. Now assume asynchrony. 
Consider the prefix $ab;ac;bc;ad$ of $\sigma$. 
Note that $\sigma \asynch_a ab;ac;bc;ad$, as $a$ is not involved in the final two calls. 
Observe that after $ab;ac;bc;ad$ agents $b$ and $c$ do not know the secret of $d$ ($ab;ac;bc;ad \models \neg S_b d \land \neg S_c d$), so that $ab;ac;bc;ad \not \models \Exp_A$. 
From that and $ab;ac;bc;ad;db;dc \asynch_a ab;ac;bc;ad$ it follows that $\sigma \not\models K_a \Exp_A$, which implies $\sigma \not\models E\!\Exp_A$, so that again $\sigma$ is not super-successful. 

We only used $\CMO$-permitted call sequences in the argument. It therefore also demonstrates that asynchronous known $\CMO$ is not super-successful (Corollary~\ref{cor:async-known-CMO}).
\end{example}
We will now show that synchronous known $\CMO$ is super-successful. 

\begin{theorem}\label{thm:syn-k-cmo-sups}
Synchronous known $\CMO$ is super-successful.
\end{theorem}
\begin{proof}
The extension of $\CMO$ consists of finite call sequences of length at most $\binom{n}{2}$. Consider a maximal $\CMO$ call sequence $\sigma$. Then we must have $|\sigma| = \binom{n}{2}$. We now use that $\CMO$ is successful, i.e., for goal $\Exp_A$~\cite{hvdetal.dynamicgossip:2019}. As there are no call sequences of length greater than $\binom{n}{2}$, and as $\CMO$ is successful, all sequences of length $\binom{n}{2}$ satisfy $\Exp_A$. As the setting is synchronous, given $\sigma$, all agents only consider call sequences of that length. Therefore, regardless of the epistemic relations, they only consider call sequences satisfying $\Exp_A$. Therefore we have $E^\CMO\Exp_A$ and $\sigma$ is super-successful.
\end{proof}

\newcommand{\overview}[2]{
  \footnotesize
  \begin{tabular}{c@{.}c@{.}c@{.}c}
    #1 \\ #2
  \end{tabular}
}

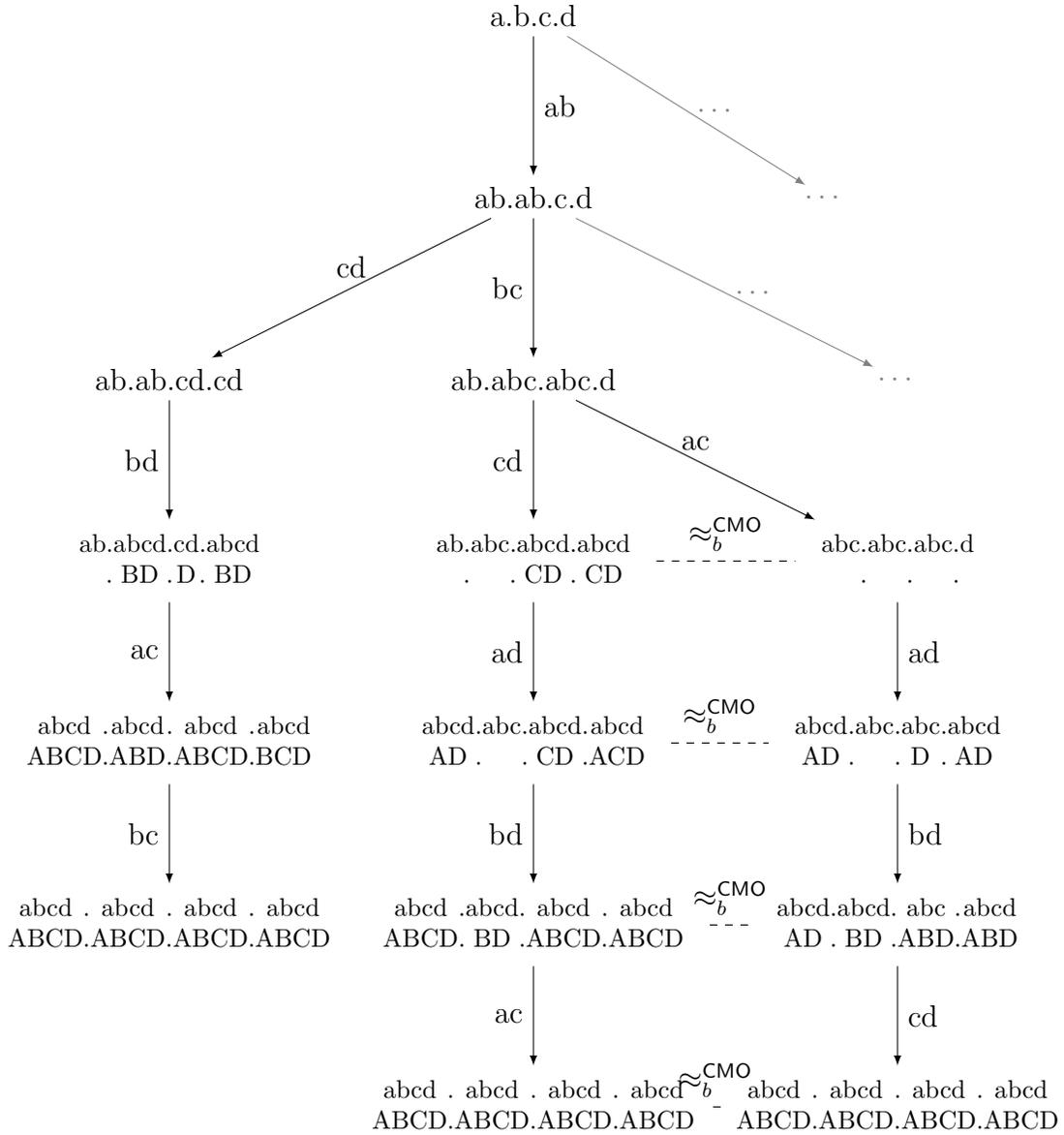
\begin{figure}
  \centering
  \begin{tikzpicture}[>=latex,node distance=2.5cm]
    \node (init) {a.b.c.d};
    \node (ab) [below of=init] {ab.ab.c.d};
    \node (o) [right of=ab, node distance=4cm,gray] {\ldots};
    \draw (init) edge [->,gray] node [right] {\ldots} (o);
    \node (ab-bc) [below of=ab] {ab.abc.abc.d};
    \node (ab-cd) [left of=ab-bc,node distance=5cm] {ab.ab.cd.cd};
    \draw (init) edge [->] node [right] {ab} (ab);
    \draw (ab) edge [->] node [above] {cd} (ab-cd);
    \draw (ab) edge [->] node [left] {bc} (ab-bc);
    \node (ab-o) [right of=ab-bc, node distance=5cm,gray] {\ldots};
    \draw (ab) edge [->,gray] node [right] {\ldots} (ab-o);
    \node (ab-cd-bd) [below of=ab-cd] {\overview{ab&abcd&cd&abcd}{ &BD&D&BD}};
    \draw (ab-cd) edge [->] node [left] {bd} (ab-cd-bd);
    \node (ab-cd-bd-ac) [below of=ab-cd-bd] {\overview{abcd&abcd&abcd&abcd}{ABCD&ABD&ABCD&BCD}};
    \draw (ab-cd-bd) edge [->] node [left] {ac} (ab-cd-bd-ac);
    \node (ab-cd-bd-ac-bc) [below of=ab-cd-bd-ac] {\overview{abcd&abcd&abcd&abcd}{ABCD&ABCD&ABCD&ABCD}};
    \draw (ab-cd-bd-ac) edge [->] node [left] {bc} (ab-cd-bd-ac-bc);
    \node (ab-bc-cd) [below of=ab-bc] {\overview{ab&abc&abcd&abcd}{& &CD&CD}};
    \draw (ab-bc) edge [->] node [left] {cd} (ab-bc-cd);
    \node (ab-bc-cd-da) [below of=ab-bc-cd] {\overview{abcd&abc&abcd&abcd}{ AD& &CD&ACD}};
    \draw (ab-bc-cd) edge [->] node [left] {ad} (ab-bc-cd-da);
    \node (ab-bc-cd-da-bd) [below of=ab-bc-cd-da] {\overview{abcd&abcd&abcd&abcd}{ ABCD&BD&ABCD&ABCD}};
    \draw (ab-bc-cd-da) edge [->] node [left] {bd} (ab-bc-cd-da-bd);
    \node (ab-bc-cd-da-bd-ac) [below of=ab-bc-cd-da-bd] {\overview{abcd&abcd&abcd&abcd}{ABCD&ABCD&ABCD&ABCD}};
    \draw (ab-bc-cd-da-bd) edge [->] node [left] {ac} (ab-bc-cd-da-bd-ac);
    \node (ab-bc-ac) [right of=ab-bc-cd,node distance=5cm] {\overview{abc&abc&abc&d}{ & & & }};
    \draw (ab-bc) edge [->] node [above] {ac} (ab-bc-ac);
    \node (ab-bc-ac-ad) [below of=ab-bc-ac] {\overview{abcd & abc  & abc & abcd}{AD & & D & AD}};
    \draw (ab-bc-ac) edge [->] node [right] {ad} (ab-bc-ac-ad);
    \node (ab-bc-ac-ad-bd) [below of=ab-bc-ac-ad] {\overview{abcd & abcd & abc & abcd}{AD &  BD & ABD & ABD}};
    \draw (ab-bc-ac-ad) edge [->] node [right] {bd} (ab-bc-ac-ad-bd);
    \node (ab-bc-ac-ad-bd-cd) [below of=ab-bc-ac-ad-bd] {\overview{abcd & abcd & abcd & abcd}{ABCD & ABCD & ABCD & ABCD}};
    \draw (ab-bc-ac-ad-bd) edge [->] node [right] {cd} (ab-bc-ac-ad-bd-cd);
    \draw (ab-bc-cd) edge [dashed] node [above] {$\synch^\CMO_b$} (ab-bc-ac);
    \draw (ab-bc-cd-da) edge [dashed] node [above] {$\synch^\CMO_b$} (ab-bc-ac-ad);
    \draw (ab-bc-cd-da-bd) edge [dashed] node [above] {$\synch^\CMO_b$} (ab-bc-ac-ad-bd);
    \draw (ab-bc-cd-da-bd-ac) edge [dashed] node [above] {$\synch^\CMO_b$} (ab-bc-ac-ad-bd-cd);
  \end{tikzpicture}
  \caption{A \textbf{partial} view of the CMO execution tree for four agents. If the first two calls are disjoint, success is (always) reached after five calls. Otherwise, it is (always) after six calls. Two other branches are suggested at depths 0 and 1 of the tree, but most other branches are not depicted. All such branches are similar to the three branches shown. In particular, after $ab;bc$ call $bd$ (or $db$) can be made, so that the same agent, $b$, occurs in the first three calls. Such a sequence therefore also succeeds after six calls.}\label{cmo.figure}
\end{figure}

\begin{example}\label{ex.goodcmo}
This example features synchronous known $\CMO$. The results in this example have been validated with the model checker GoMoChe. They are displayed in Tables~\ref{tab:fiveCalls} and~\ref{tab:sixCalls}, and in Figure~\ref{cmo.figure}.

Given four agents $a,b,c,d$, we always reach $E^\CMO\!\Exp_A$ in five calls when the first two calls have no overlap, as in the prefix $ab;cd$ of the five-call sequence $ab;cd;bd;ac;bc$ displayed in Figure~\ref{cmo.figure}. The only $\CMO$-permitted call that has not yet been made after this five-call sequence is $ad$.  

Given synchrony it is not always obvious how agents not involved in a call  learn that agents become super experts in that call. We will therefore justify in detail how this may come to pass for some agents.

For example, in third call $bd$ agent $c$ learns that $d$ becomes a super expert. This is because in the second call $cd$, agent $c$ learns that the first call was $ab$, and as $c$ is not involved in the third call, this must be one of $ab,ad,bd$ (or the dual call). As $c$ knows that $ab$ has already taken place, the third call must therefore have been between $a$ and $d$ or between $b$ and $d$. This always involves $d$, and $d$ then always becomes an expert. Therefore, $c$ knows that $d$ is an expert.

Similarly, in the fifth call $bc$, agent $d$ becomes a super expert (and in particular learns that $a$ is an expert), because $d$ knows that the two remaining $\CMO$-permitted calls were $bc$ and $ad$. As $d$ was not involved, $d$ knows that the call was $bc$.

\begin{table}[ht]
  \centering
  \begin{tabular}{l|l@{}l@{}l@{}l@{ }l@{}l@{}l@{}l|l@{}l@{}l@{}l@{ }l@{}l@{}l@{}l|l@{}l@{}l@{}l@{ }l@{}l@{}l@{}l|l@{}l@{}l@{}l@{ }l@{}l@{}l@{}l|l}
    \           & a& & &  &  & & &  &  &b& &  &   & & &  &   & &c&  &  & & &  &  & & &d &  & & &  \\
    $\call{ab}$ & a&b& &  &  & & &  & a&b& &  &   & & &  &   & &c&  &  & & &  &  & & &d &  & & &  \\
    $\call{cd}$ & a&b& &  &  & & &  & a&b& &  &   & & &  &   & &c&d  &  & & &  &  & &c&d &  & & &  \\
    $\call{bd}$ & a&b& &  &  & & &  & a&b&c&d  &   &B& &D &   & &c&d &  & & &D  & a&b&c&d &  &B& &D \\
    $\call{ac}$ & a&b&c&d & A&B&C&D  & a&b&c&d  &  A&B& &D &  a&b&c&d &  A&B&C&D & a&b&c&d &  &B&C&D \\
    $\call{bc}$ & a&b&c&d & A&B&C&D & a&b&c&d &  A&B&C&D &  a&b&c&d & A&B&C&D & a&b&c&d & A&B&C&D \\
  \end{tabular}
  \caption{The results of $ab;cd;bd;ac;bc$ with synchronous known $\CMO$.}\label{tab:fiveCalls}
\end{table}

However, if we start with overlapping calls $ab;bc$, then $E^\CMO\!\Exp_A$ is only reached after $6$ calls. 
For example, consider the sequence $ab;bc;cd;ad;bd;ca$. 
After this sequence everyone is a super expert. We show the results of this sequence in Table~\ref{tab:sixCalls}.

\begin{table}[ht]
  \centering
  \begin{tabular}{l|l@{}l@{}l@{}l@{ }l@{}l@{}l@{}l|l@{}l@{}l@{}l@{ }l@{}l@{}l@{}l|l@{}l@{}l@{}l@{ }l@{}l@{}l@{}l|l@{}l@{}l@{}l@{ }l@{}l@{}l@{}l|l}
    \           & a& & &  &  & & &  &  &b& &  &   & & &  &   & &c&  &  & & &  &  & & &d &  & & &  \\
    $\call{ab}$ & a&b& &  &  & & &  & a&b& &  &   & & &  &   & &c&  &  & & &  &  & & &d &  & & &  \\
    $\call{bc}$ & a&b& &  &  & & &  & a&b&c&  &   & & &  &  a&b&c&  &  & & &  &  & & &d &  & & &  \\
    $\call{cd}$ & a&b& &  &  & & &  & a&b&c&  &   & & &  &  a&b&c&d &  & &C&D & a&b&c&d &  & &C&D \\
    $\call{ad}$ & a&b&c&d & A& & &D & a&b&c&  &   & & &  &  a&b&c&d &  & &C&D & a&b&c&d & A& &C&D \\
    $\call{bd}$ & a&b&c&d & A&B&C&D & a&b&c&d &   &B& &D &  a&b&c&d & A&B&C&D & a&b&c&d & A&B&C&D \\
    $\call{ac}$ & a&b&c&d & A&B&C&D & a&b&c&d &  A&B&C&D &  a&b&c&d & A&B&C&D & a&b&c&d & A&B&C&D \\
  \end{tabular}
  \caption{The results of $ab;bc;cd;ad;bd;ac$ with synchronous known $\CMO$.}\label{tab:sixCalls}
\end{table}

After the five calls $ab;bc;cd;ad;bd$ agent $b$ considers $ab;bc;ac;ad;bd$ possible, after which $c$ is not an expert. But $b$ has already been in a call with each other agent, and hence $b$ is no longer $\CMO$-permitted to make calls. However, call $ac$ has not yet been made. Although agent $a$ is a super expert, call $ac$ is $\CMO$-permitted, after which the protocol terminates super-successfully.
\end{example}

If in the above Example~\ref{ex.goodcmo} the horizon of six calls has been reached, it is even \emph{common knowledge}\footnote{Common knowledge is an infinitary epistemic notion proposed in, for example,~\cite{lewis:1969,aumann:1976,halpernmoses:1990}} that all agents are experts, and thus it is common knowledge that they are super experts. However, if termination is earlier, after five calls, we are uncertain if such common knowledge is then reached. We conjecture that it is, and that this might be the case for $\CMO$ for any number $n \geq 4$ of agents, and possibly for other gossip protocols as well that do not have infinite executions. We wish to add common knowledge to the logical language in future research and investigate these matters.
The common knowledge operator was already studied in the context of gossip protocols in~\cite{AptW17}, but not for the super-successful goal.

\section{Engaged agents}\label{section:engaged}

\subsection{Syntax and semantics --- engaged agents}

We now provide a semantics wherein agents who are super experts do not make calls and do not answer calls. In principle, these are independent features that could be modelled separately, but we model them jointly. The idea behind this is utilitarian. Each agent's principal interest is to get to know all the secrets. Knowing them all, why continue to pay attention to further calls? The agent will thus be tempted to `walk away from the scene', which means no longer making and no longer answering calls. As we have already seen in the introductory section, this might prevent other agents from becoming experts. An agent valuing that everybody is an expert would first ascertain (or, in that way, ensure) that everybody else is also an expert, and only then walk away. In other words, she walks away once she is a super expert. So only then, she is no longer making and no longer answering calls. So, we model these aspects jointly.

To model that agents who are super experts do not \emph{make} calls, we need to change the definition of gossip protocol. An extra condition to be permitted to call is that the agent is not a super expert. In other words, we strengthen the protocol condition.

To model that agents who are super experts do not \emph{answer} calls, we need to change the definition of the epistemic relation, such that a call sequence cannot be extended with a call made by a super expert. In other words, we change the meaning of knowledge, as this a function of the epistemic relation, and therefore we indirectly also change the termination condition.

Agents who neither make nor answer calls are called \emph{engaged agents} (as in `engaged in other activities' for the former and as in `the line is engaged' for the latter). A call that is not answered is a \emph{missed call}.
Giving meaning to a missed call is quite common, see for example~\cite{Donner2007:Beeping} for an interview survey and~\cite{Dogar2020:MissIt} for a general ``one bit per second'' protocol solely using missed calls.
Letting the telephone ring until a connection is made is free, whereas making the connection and having a conversation, however short, is payable. Communication by missed calls is therefore free of charge. Such communication is only meaningful if there is common knowledge of the meaning of a missed call, just as in our case. (Although, unlike here, this is then typically informative for the callee, and not for the caller.) Obviously the telephone companies do not like that unintended use.

The extension involving engaged agents should be seen as a form of common knowledge of the gossip protocol. Super experts not making calls can already be taken care of in the definition of the protocol, by strengthening the calling condition, whereas super experts not answering calls require changing the epistemic relation, such that it really concerns some form of common knowledge of the gossip protocol. We present the engaged agents feature as an extension of any given protocol ``known $\prot$'', so one where $\prot$ is also common knowledge. We will call such a protocol ``\emph{protocol known $\prot$ with engaged agents}''. A semantics for engaged agents where $\prot$ is not common knowledge would be conceivable. In that case, recalling that known $\ANY$ is the same as $\ANY$, consider $\ANY$ with engaged agents, and apply $\prot$ on the set of thus permitted call sequence.

\begin{definition}[Protocol --- engaged agents]\label{def.prot3}
A protocol $\prot$ (denoted ``known $\prot$ with engaged agents'') is a program defined by
 \[ \prot := {(\Union_{a\neq b \in A} (?(\neg K_a^\prot \Exp_A \et \prot_{ab});ab))}^* ; ? E^\prot\!\Exp_A \]
where for all $a\neq b \in A$, $\prot_{ab} \in \lang$ is the \emph{protocol condition} for call $ab$ of protocol $\prot$. 
\end{definition}

This protocol definition is different from the previous Definitions~\ref{def.prot} and~\ref{def.prot2} but also different from the usual definition (e.g.,~\cite{DitmarschGKP19}):
 \[ {(?\neg \Exp_A ; \Union_{a\neq b \in A} (?\prot_{ab};ab))}^* ; ? \Exp_A \]
As our termination condition is stronger, we already replaced ``while not everyone is an expert'' by ``while not everyone is a super expert'' and the protocol becomes~Definition~\ref{def.prot3}:
\[ {(?\neg E^\prot\!\Exp_A ; \Union_{a\neq b \in A} (?\prot_{ab};ab))}^* ; ? E^\prot\!\Exp_A \]
Then, as we do not want super experts to make calls, we strengthen the protocol condition by adding $\lnot K_a^P \Exp_A$ to it:
 \[ {(?\neg E^\prot\!\Exp_A ; \Union_{a\neq b \in A} (?(\neg K_a^\prot \Exp_A\et\prot_{ab});ab))}^* ; ? E^\prot\!\Exp_A \]
Finally, as $\Et_{a \in A} K_a^\prot \Exp_A$ is $E^\prot\!\Exp_A$, it is easy to see that the same call sequences are permitted if we remove the first test on $\neg E^\prot\!\Exp_A$, which leads to the above Definition~\ref{def.prot3}. 

We continue with the changed epistemic relations. The definition of the semantic relation $\models$ remains the same. In particular, the semantics of an actually made call $ab$ is the same as the semantics of a missed call $ab$. Although no secrets are exchanged in a missed call, if the secrets had been exchanged neither agent would have learnt a new secret, as both agents must already have been experts. (So, formally, when $S = A^2$ then for all $x,y \in A$, $S^\sigma = (\{(x,y),(y,x)\} \circ S^\sigma)$, so that $S^{\sigma;ab} = S^\sigma \cup (\{(a,b),(b,a)\} \circ S^\sigma) = S^\sigma$.)

\begin{definition}[Epistemic relation --- engaged agents]\label{def.synchEngaged} \ \\
Let $a\in A$. The synchronous accessibility relation $\synch_a^\prot$ between call sequences is the smallest symmetric and transitive relation such that:
\begin{itemize}
\item $\epsilon \synch_a^\prot \epsilon$,
\item if $\sigma \synch_a^\prot \tau$,
 $a \notin \{b,c,d,e\}$,
 $\sigma \models \neg K_b^\prot \Exp_A \et \prot_{bc}$ and
 $\tau  \models \neg K_d^\prot \Exp_A \et \prot_{de}$
 then
 $\sigma;bc \synch_a^\prot \tau;de$
\item if $\sigma \synch_a^\prot \tau$,
 $I^\sigma_b = I^\tau_b$,
 $\sigma \models \neg K_a^\prot \Exp_A \et \prot_{ab}$,
 $\tau  \models \neg K_a^\prot \Exp_A \et \prot_{ab}$, and
 ($\sigma \models K_b^\prot \Exp_A$ iff $\tau \models K_b^\prot \Exp_A$),
 then
 $\sigma;ab \synch_a^\prot \tau;ab$
\item if $\sigma \synch_a^\prot \tau$,
 $I^\sigma_b = I^\tau_b$,
 $\sigma \models \neg K_b^\prot \Exp_A \et \prot_{ba}$,
 $\tau  \models \neg K_b^\prot \Exp_A \et \prot_{ba}$,
 and ($\sigma \models K_a^\prot \Exp_A$ iff $\tau \models K_a^\prot \Exp_A$),
 then
 $\sigma;ba \synch_a^\prot \tau;ba$
\end{itemize}
The asynchronous accessibility relation $\asynch_a^\prot$ between gossip states is as the relation $\synch_a^\prot$ except that the second clause is replaced by
\begin{itemize}
\item if $\sigma \asynch_a^\prot \tau$,
 $a \notin \{b,c\}$, and
 $\sigma \models \neg K_b^\prot \Exp_A \land \prot_{bc}$,
 then
 $\sigma;bc \asynch_a^\prot \tau$
\end{itemize}
\end{definition}

In the first place, the above definitions incorporate that agents no longer make calls once they are super experts. This is the part $\neg K_a^\prot \Exp_A$ in the definition of protocol, and the parts $\neg K_a^\prot \Exp_A$ and $\neg K_b^\prot \Exp_A$ in respectively the third and fourth item of Definition~\ref{def.synchEngaged} of the epistemic relation. 

In the second place, the extra conditions ``$\sigma \models K_b^\prot \Exp_A$ iff $\tau \models K_b^\prot \Exp_A$'' and ``$\sigma \models K_a^\prot \Exp_A$ iff $\tau \models K_a^\prot \Exp_A$''  in the third and fourth items of the definition of the epistemic relation, model that agents $b$ and $a$, respectively, no longer answer calls once they are super experts. 
For example, in the third item it has the effect that after a missed call $ab$, any state $\tau$ after which $ab$ is not a missed call ($b$ answers the call) is no longer considered possible by agent $a$. 
In other words, we then have that $\sigma;ab \not\synch_a \tau;ab$, so that after $\sigma;ab$ agent $b$ \emph{knows} that $ab$ was a missed call.

The properties of protocol-dependent knowledge $K_a^\prot$ listed in the previous section also hold for the semantics extended with the feature of engaged agents. In particular, on the set of all call sequences that are $\prot$-permitted and such that super experts do not make calls, the relations $\synch_a^\prot$ and $\asynch_a^\prot$ are equivalence relations. 

A special feature of the semantics with engaged calls is that calling a super expert will also make the callee a super expert:

\begin{lemma}\label{prop.missed}
  In the semantics with engaged calls, $\models K_b^\prot \Exp_A \imp [ab]K_a^\prot \Exp_A$.
\end{lemma} 
\begin{proof}
  We give the proof for the asynchronous epistemic relation. The proof is similar for the synchronous relation.
  Let $\sigma \models K_b^\prot \Exp_A$ and assume $\sigma\models \neg K_a^\prot \Exp_A \et \prot_{ab}$. Let $\tau'$ be such that $\sigma;ab \asynch_a \tau'$. Given the definition of the epistemic relation, $\tau' = \tau;ab$ for some $\tau$, and from $\sigma;ab \asynch_a^\prot \tau;ab$ we also obtain $\sigma \asynch_a^\prot \tau$.
  As $\sigma \models K_b^\prot \Exp_A$ and $\sigma \asynch_a^\prot \tau$, from the definition of the epistemic relation we obtain  $\tau \models K_b^\prot \Exp_A$, and thus also $\tau;ab \models K_b^\prot \Exp_A$. As knowledge is correct after $\prot$-permitted sequences (Section~\ref{section:ComKnowP}), also $\tau;ab \models \Exp_A$. And as $\tau$ was arbitrary such that $\sigma;ab \asynch_a^\prot \tau;ab$, we obtain $\sigma;ab \models K_a^\prot \Exp_A$ and thus $\sigma \models [ab]K_a^\prot \Exp_A$ as desired.
\end{proof}

The dual effect of this semantics for engaged calls is, that when after $\sigma$ agent $b$ answers a call from $a$, any state $\tau$ wherein agent $b$ would have been a super expert is no longer considered possible by $a$.  In particular, even when $a$ learns that $b$ already knew all secrets before the call $ab$, she learns that $b$ was not yet a super expert after $\sigma$. 
Of course, $b$ may have become a super expert in the call $ab$. 

\subsection{Results for the protocols ANY and CMO}\label{section:resultsany2}

We show that $\ANY$ is super-successful for the semantics with engaged agents, whereas $\CMO$ is now no longer super-successful. For $\PIG$ this is unknown. We also present a result for execution length.
\begin{theorem}\label{thm:anyengaged}
$\ANY$ with engaged agents is super-successful.
\end{theorem}
\begin{proof}
The proof consists of a slight adaptation of the proof of Theorem~\ref{Thm:pig}. For the protocol $\ANY$, the part $\neg K_a^\prot \Exp_A \et \prot_{ab}$ of the engaged agents protocol definition becomes $\neg K_a \Exp_A$. This is equivalent to $\M_a \Vel_{c \in A} \neg S_b c$. This is a slight weakening of $\PIG_{ab}$ which is $\M_a \Vel_{c \in A} (S_a c \et \neg S_b c) \vel (\neg S_a c \et S_b c)$. A simplification of the argument in the proof of Theorem~\ref{Thm:pig} now suffices, where (i) we do not need to consider the case $\neg S_a c \et S_b c$ given the simpler condition to be satisfied, and (ii) calls by super experts do not occur in an (engaged $\ANY$)-permitted call sequence. We can think of such disallowed calls as being removed from the assumed $\PIG$-permitted call sequences in the proof of Theorem~\ref{Thm:pig}, where we note that the calls $de$ needed to reach super-successful termination must satisfy $\M_d \neg S_e f$ (see the penultimate paragraph of the proof of Theorem~\ref{Thm:pig}), such that $d$ is therefore not a super expert and therefore $de$ is (engaged $\ANY$)-permitted.

Again, the proof holds for synchrony and asynchrony.
\end{proof}
However, it remains unclear whether known $\PIG$ with engaged agents is also super-successful.  In this case, the part $\neg K_a^\prot \Exp_A \et \prot_{ab}$ of the engaged agents protocol definition becomes $\neg K_a^\PIG \Exp_A\et \PIG_{ab}$, which is equivalent to $\M_a^\PIG \Vel_{c \in A} \neg S_b c$, that is, $\Vel_{c \in A} \M_a^\PIG  \neg S_b c$. We do not know whether $\M_a \neg S_b c$ is equivalent to $\M^\PIG_a \neg S_b c$.

\medskip

We continue with some results for asynchronous $\ANY$ demonstrating how the feature of engaged agents affects termination.

\begin{example}
Consider again Example~\ref{ex.threea} for three agents $a,b,c$ and super-successful call sequence $ab;ac;ab;cb$. With engaged agents, final call $cb$ is a missed call. The sequence remains super-successful. But we need that final call.
\end{example}

\begin{example}\label{ex.15}
Given are six agents $a,b,c,d,e,f$. We first assume asynchronous $\ANY$ without engaged agents. We enact the procedure also used in the proof of Proposition~\ref{prop.minmin}. A standard solution to obtain $\Exp_A$ is $ae;af;ab;cd;ac;bd;ae;af$. It consists of \emph{eight} calls. After any of the final \emph{four} calls $ac;bd;ae;af$, the involved agents are experts. The agents can continue to verify that all other agents are experts in subsequent calls. Altogether this requires {\bf each pair} of agents to make a call after which they both are (or remain) experts. For $6$ agents we therefore need $8-4+ \binom{6}{2} = 4+15=19$ calls. An example execution with all calls in lexicographic order is as follows.
 \[ ae;af;ab;cd;ac;bd;ae;af;ab;ad;bc;be;bf;cd;ce;cf;df;ed;ef \]
With engaged agents, a simpler sequence with $15$ instead of $19$ calls is already super-successful:
 \[ ae;af;ab;cd;ac;bd;ae;af;ab;ad;ba;ca;da;ea;fa \]
 In this sequence first $a$ becomes a super expert, in call $ad$. Then all other agents call agent $a$. These are the final five calls $ba;ca;da;ea;fa$, These are therefore all missed calls in which $b$ to $f$ also become super experts. 
\end{example}

In Proposition~\ref{prop.minmin} we showed that with $n$ agents, super-successful asynchronous $\ANY$ termination is reached in $n - 2 + \binom{n}{2}$ calls, which is of $\bigO(n^2)$ complexity. We now show (or rather observe) that asynchronous $\ANY$ with engaged agents termination is reached in $3n-4$ calls, which is of $\bigO(n)$ complexity. We conjecture that the bound $3n-4$ is minimal.

\begin{proposition}\label{prop.minmin2}
Given $n$ agents, super-successful asynchronous $\ANY$ termination with engaged agents can be achieved in $3n-4$ calls.
\end{proposition}
\begin{proof}
Select an agent $a$ among the $n$ agents. First, agent $a$ calls all other agents. These are $n-1$ calls. Then, agent $a$ calls all agents again in the same order, except the last one that was called in the first round. These are $n-2$ calls. Finally, all other agents call $a$. These are $n-1$ calls. Altogether these are $3n-4$ calls. The final $n-1$ calls are all missed calls. After a missed call the calling agent is also a super expert (Lemma~\ref{prop.missed}). All agents are then super experts: $E\Exp_A$ holds. 
\end{proof}
The proof extends the method used in Proposition~\ref{ex.threes} to show that $2n-3$ calls are enough to make a super expert. We merely add another $n-1$ calls to the $2n-3$ already made.

\begin{example} 
For four agents, the method also used the proof of Proposition~\ref{prop.minmin2} constructs a $14$-call super-successful call sequence (so, one less than in Example~\ref{ex.15} above). All calls involve $a$. First, $a$ calls everyone else, then $a$ calls everyone else except the last agent $f$, finally everyone else calls $a$, all of which are missed calls. We obtain: \[ ab;ac;ad;ae;af; \ \  ab;ac;ad;ae; \ \  ba;ca;da;ea;fa \] 
\end{example}

So far, all the news involving engaged agents seems good: speedier termination. We close with a bit of bad news. When engaged agents withdraw from the conversation this can impede dissemination of information, and even prevent that super-successful termination. We recall Theorem~\ref{thm:syn-k-cmo-sups} that synchronous known $\CMO$ is super-successful. Unfortunately, with engaged agents it is no longer super-successful.

\begin{proposition}\label{thm:ck-and-sync-CMO}
Synchronous known $\CMO$ with engaged agents is not super-successful.
\end{proposition}
\begin{proof}
The proof is by counterexample. Consider again Example~\ref{ex.goodcmo} and Table~\ref{tab:sixCalls}. Consider (prefix) sequence $ab;bc;cd;ad;bd$. After this sequence everyone but $b$ is a super expert.

Agent $b$ considers $ab;bc;ac;ad;bd$ possible (see again  Figure~\ref{cmo.figure}) after which $c$ is not an expert. But $b$ has already been in a call with each other agent, and hence $b$ is no longer permitted to make calls. On the other hand, agents $a$ and $c$ have not been in a call yet, so $ac$ and $ca$ are $\CMO$-permitted, but they are both super experts (see Table~\ref{tab:fiveCalls}) and will therefore not make a call. The protocol terminates unsuccessfully.
\end{proof}

A fortiori this holds for asynchrony, as agent $b$ then considers it possible that the call $bd$ was the last call. Therefore:
\begin{corollary}\label{cor:ck-and-sync-CMO}
Asynchronous known $\CMO$ with engaged agents is not super-successful.
\end{corollary}

If only, in Example~\ref{ex.goodcmo}, agent $b$ could be sure that after the sequence $ab;bc;ac;ad;bd$ the final call $cd$ would be made \dots \ But even though we assume synchrony, $b$ knows nothing about the interval between calls and therefore $b$ cannot become a super expert. To become an expert $b$ would have to reason as follows: \begin{quote} I am uncertain between two call sequences. An interval $x$ of time has now passed (a `clock tick'). After one sequence another call was permitted, after which all agents are experts. After the other sequence no call was permitted, but all agents already were experts. Therefore, I now know that all agents are experts.\end{quote}
By another extension of the semantics with explicit `clock ticks', using basic programs called `\Skip', we can still make $\CMO$ super-successful. This rather technical extension is   presented in Appendix B, as a further illustration how versatile and general our method is to adapt the logical semantics of epistemic gossip and to obtain PDL-style results for it. It involves changing the logical language, because we add a basic program called $\Skip$. The addition of $\Skip$ only makes sense in the presence of time, so, we must assume synchrony. Such an addition seems already of interest without the semantics with common knowledge of protocols and engaged agents, and for other protocols than $\CMO$.

\section{Conclusion and further research}\label{section:conclusion}

We explored gossip protocols wherein the termination condition is that all agents know that all agents know all secrets. Such agents are called super experts and call sequence satisfying that is called super-successful. For our results it matters whether the agents have common knowledge which gossip protocol is executed. For protocols with this epistemic goal we also investigated what happens when agents who are super experts do not make and do not answer calls. Such agents are called engaged agents. 

We investigated conditions under which such gossip protocols terminate, both in the synchronous case, where there is a global clock, and in the asynchronous case, where there is not. In particular, the protocol $\CMO$ wherein agents may only call each other once, is super-successful in the presence of a global clock. We further showed that synchronous protocols may terminate before asynchronous protocols, and that protocols with engaged agents may terminate before protocols without. 

Table~\ref{table:overview} provides an overview of our results.

\begin{table}[H]
  \centering
  \begin{tabular}{l|lll}
    \toprule
    \               & standard $\prot$        & known $\prot$                & known $\prot$ + engaged      \\
    \midrule
    $\ANY$          & YES, \cref{prop.anysuper} & YES, \cref{prop.anysuper}      & YES, \cref{thm:anyengaged}     \\[0.5em]
     $\CMO\synch$   & NO, \cref{thm:nonck-CMO}  & YES, \cref{thm:syn-k-cmo-sups} & NO, \cref{thm:ck-and-sync-CMO} \\
      $\CMO\asynch$ & NO, \cref{thm:async-CMO}  & NO, \cref{cor:async-known-CMO} & NO, \cref{cor:ck-and-sync-CMO} \\[0.5em]
    $\PIG$          & YES, \cref{Thm:pig}       & YES, \cref{thm:xxx}            & ?                            \\[0.5em]
    \bottomrule
  \end{tabular}
  \caption{Overview which protocols are super-successful under which semantics.}\label{table:overview}
\end{table}

The gap in Table~\ref{table:overview} is left for future research. Our results appear to generalize to gossip protocols with the termination condition that it is common knowledge that all agents are experts, which seems also worth to investigate later. It may finally be of interest to investigate gossip protocols with very different epistemic calling conditions.

\paragraph{Acknowledgements}
This work is loosely based on a contribution~\cite{RRvDG2020:EveryoneKnows} with a similar title and one additional author (Rasoul Ramezanian), for a volume honouring Mohammad Ardeshir at his retirement.
We thank the reviewers of previous versions of this article, and in particular the reviewer at \emph{Studia Logica} for their helpful corrections and detailed comments.

\bibliographystyle{plainurl}
\bibliography{termGossip}

\begin{thebibliography}{10}

\bibitem{apt-tark}
K.R. Apt, D.~Grossi, and W.~van~der Hoek.
\newblock Epistemic protocols for distributed gossiping.
\newblock In {\em Proceedings of 15th TARK}, pages 51--66, 2015.
\newblock \href {https://doi.org/10.4204/EPTCS.215.5}
  {\path{https://doi.org/10.4204/EPTCS.215.5}}.

\bibitem{AptGH18}
K.R. Apt, D.~Grossi, and W.~van~der Hoek.
\newblock When are two gossips the same?
\newblock In G.~Barthe, G.~Sutcliffe, and M.~Veanes, editors, {\em Proc.\ of
  22nd {LPAR}}, volume~57 of {\em EPiC Series in Computing}, pages 36--55,
  2018.
\newblock \href {https://doi.org/10.29007/ww65}
  {\path{https://doi.org/10.29007/ww65}}.

\bibitem{AptW17}
K.R. Apt and D.~Wojtczak.
\newblock Common knowledge in a logic of gossips.
\newblock In J.~Lang, editor, {\em Proc.\ of the 16th {TARK}}, volume 251 of
  {\em {EPTCS}}, pages 10--27, 2017.
\newblock \href {https://doi.org/10.4204/EPTCS.251.2}
  {\path{https://doi.org/10.4204/EPTCS.251.2}}.

\bibitem{AptW18}
K.R. Apt and D.~Wojtczak.
\newblock Verification of distributed epistemic gossip protocols.
\newblock {\em J. Artif. Intell. Res.}, 62:101--132, 2018.
\newblock \href {https://doi.org/10.1613/jair.1.11204}
  {\path{https://doi.org/10.1613/jair.1.11204}}.

\bibitem{abs-1907-09097}
K.R. Apt and D.~Wojtczak.
\newblock Open problems in a logic of gossips.
\newblock In L.S. Moss, editor, {\em Proc.\ of the 17th {TARK}}, volume 297 of
  {\em {EPTCS}}, pages 1--18, 2019.
\newblock \href {https://doi.org/10.4204/EPTCS.297.1}
  {\path{https://doi.org/10.4204/EPTCS.297.1}}.

\bibitem{attamahetal.ecai:2014}
M.~Attamah, H.~van Ditmarsch, D.~Grossi, and W.~van~der Hoek.
\newblock Knowledge and gossip.
\newblock In {\em Proc.\ of 21st {ECAI}}, pages 21--26. {IOS} Press, 2014.
\newblock \href {https://doi.org/10.3233/978-1-61499-419-0-21}
  {\path{https://doi.org/10.3233/978-1-61499-419-0-21}}.

\bibitem{Attamah2017}
M.~Attamah, H.~van Ditmarsch, D.~Grossi, and W.~van~der Hoek.
\newblock The pleasure of gossip.
\newblock In C.~Ba{\c{s}}kent, L.S. Moss, and R.~Ramanujam, editors, {\em Rohit
  Parikh on Logic, Language and Society}, pages 145--163. Springer, 2017.
\newblock \href {https://doi.org/10.1007/978-3-319-47843-2_9}
  {\path{https://doi.org/10.1007/978-3-319-47843-2_9}}.

\bibitem{aumann:1976}
R.J. Aumann.
\newblock Agreeing to disagree.
\newblock {\em Annals of Statistics}, 4(6):1236--1239, 1976.
\newblock URL: \url{https://www.jstor.org/stable/2958591}.

\bibitem{BAKER1972191}
B.~Baker and R.~Shostak.
\newblock Gossips and telephones.
\newblock {\em Discrete Mathematics}, 2(3):191--193, 1972.
\newblock \href {https://doi.org/10.1016/0012-365X(72)90001-5}
  {\path{https://doi.org/10.1016/0012-365X(72)90001-5}}.

\bibitem{boydSteele1979}
D.~Boyd and J.~Steele.
\newblock Random exchanges of information.
\newblock {\em Journal of Applied Probability}, 16:657--661, 1979.
\newblock \href {https://doi.org/10.2307/3213094}
  {\path{https://doi.org/10.2307/3213094}}.

\bibitem{CooperHMMR19}
M.C. Cooper, A.~Herzig, F.~Maffre, F.~Maris, and P.~R{\'{e}}gnier.
\newblock The epistemic gossip problem.
\newblock {\em Discret. Math.}, 342(3):654--663, 2019.
\newblock \href {https://doi.org/10.1016/j.disc.2018.10.041}
  {\path{https://doi.org/10.1016/j.disc.2018.10.041}}.

\bibitem{doerr}
B.~Doerr, T.~Friedrich, and T.~Sauerwald.
\newblock Quasirandom rumor spreading.
\newblock {\em ACM Trans. Algorithms}, 11(2):1--35, 2014.
\newblock \href {https://doi.org/10.1145/2650185}
  {\path{https://doi.org/10.1145/2650185}}.

\bibitem{Dogar2020:MissIt}
F.~Dogar, I.~Qazi, A.~Tariq, G.~Murtaza, A.~Ahmad, and N.~Stocking.
\newblock Missit: Using missed calls for free, extremely low bit-rate
  communication in developing regions.
\newblock In {\em Proc.\ of the 2020 CHI Conference on Human Factors in
  Computing Systems}, pages 1--12, 2020.
\newblock \href {https://doi.org/10.1145/3313831.3376259}
  {\path{https://doi.org/10.1145/3313831.3376259}}.

\bibitem{Donner2007:Beeping}
J.~Donner.
\newblock {The Rules of Beeping: Exchanging Messages Via Intentional “Missed
  Calls” on Mobile Phones}.
\newblock {\em Journal of Computer-Mediated Communication}, 13(1):1--22, 2007.
\newblock \href {https://doi.org/10.1111/j.1083-6101.2007.00383.x}
  {\path{https://doi.org/10.1111/j.1083-6101.2007.00383.x}}.

\bibitem{eugster}
P.~Eugster, R.~Guerraoui, A.~Kermarrec, and L.~Massouli{\'{e}}.
\newblock Epidemic information dissemination in distributed systems.
\newblock {\em {IEEE} Computer}, 37(5):60--67, 2004.
\newblock \href {https://doi.org/10.1109/MC.2004.1297243}
  {\path{https://doi.org/10.1109/MC.2004.1297243}}.

\bibitem{Haeupler15}
B.~Haeupler.
\newblock Simple, fast and deterministic gossip and rumor spreading.
\newblock {\em Journal of the {ACM}}, 62(6):47, 2015.
\newblock \href {https://doi.org/10.1145/2767126}
  {\path{https://doi.org/10.1145/2767126}}.

\bibitem{haigh81}
J.~Haigh.
\newblock Random exchanges of information.
\newblock {\em Journal of Applied Probability}, pages 743--746, 1981.
\newblock \href {https://doi.org/10.2307/3213330}
  {\path{https://doi.org/10.2307/3213330}}.

\bibitem{halpernmoses:1990}
J.Y. Halpern and Y.~Moses.
\newblock Knowledge and common knowledge in a distributed environment.
\newblock {\em Journal of the {ACM}}, 37(3):549--587, 1990.
\newblock \href {https://doi.org/10.1145/79147.79161}
  {\path{https://doi.org/10.1145/79147.79161}}.

\bibitem{hareletal:2000}
D.~Harel, D.~Kozen, and J.~Tiuryn.
\newblock {\em Dynamic Logic}.
\newblock MIT Press, Cambridge MA, 2000.
\newblock Foundations of Computing Series.

\bibitem{hedetniemietal:1988}
S.M. Hedetniemi, S.T. Hedetniemi, and A.L. Liestman.
\newblock A survey of gossiping and broadcasting in communication networks.
\newblock {\em Networks}, 18:319--349, 1988.
\newblock \href {https://doi.org/10.1002/net.3230180406}
  {\path{https://doi.org/10.1002/net.3230180406}}.

\bibitem{HerzigM17}
A.~Herzig and F.~Maffre.
\newblock How to share knowledge by gossiping.
\newblock {\em {AI} Commun.}, 30(1):1--17, 2017.
\newblock \href {https://doi.org/10.3233/AIC-170723}
  {\path{https://doi.org/10.3233/AIC-170723}}.

\bibitem{kermarrecetal:2007}
A.-M. Kermarrec and M.~van Steen.
\newblock Gossiping in distributed systems.
\newblock {\em SIGOPS Oper. Syst. Rev.}, 41(5):2--7, 2007.
\newblock \href {https://doi.org/10.1145/1317379.1317381}
  {\path{https://doi.org/10.1145/1317379.1317381}}.

\bibitem{knoedel:1975}
W.~Kn{\"o}del.
\newblock New gossips and telephones.
\newblock {\em Discrete Mathematics}, 13:95, 1975.
\newblock \href {https://doi.org/10.1016/0012-365X(75)90090-4}
  {\path{https://doi.org/10.1016/0012-365X(75)90090-4}}.

\bibitem{lewis:1969}
D.K. Lewis.
\newblock {\em Convention, a Philosophical Study}.
\newblock Harvard University Press, 1969.

\bibitem{LiveseyW22}
J.~Livesey and D.~Wojtczak.
\newblock Propositional gossip protocols under fair schedulers.
\newblock In L.~De Raedt, editor, {\em Proc.\ of 31st {IJCAI}}, pages 391--397,
  2022.
\newblock \href {https://doi.org/10.24963/ijcai.2022/56}
  {\path{https://doi.org/10.24963/ijcai.2022/56}}.

\bibitem{moon72}
J.W. Moon.
\newblock Random exchanges of information.
\newblock {\em Nieuw Archief voor Wiskunde}, 3(20):246--249, 1972.

\bibitem{RRvDG2020:EveryoneKnows}
R.~Ramezanian, R.~Ramezanian, H.~van Ditmarsch, and M.~Gattinger.
\newblock Everyone knows that everyone knows.
\newblock In M.~Mojtahedi, S.~Rahman, and M.S. Zarepour, editors, {\em
  Mathematics, Logic, and Their Philosophies: Essays in Honour of Mohammad
  Ardeshir}, 2021.
\newblock \href {https://doi.org/10.1007/978-3-030-53654-1_5}
  {\path{https://doi.org/10.1007/978-3-030-53654-1_5}}.

\bibitem{tijdeman:1971}
R.~Tijdeman.
\newblock On a telephone problem.
\newblock {\em Nieuw Archief voor Wiskunde}, 3(19):188--192, 1971.

\bibitem{DitmarschG22}
H.~van Ditmarsch and M.~Gattinger.
\newblock The limits to gossip: Second-order shared knowledge of all secrets is
  unsatisfiable.
\newblock In A.~Ciabattoni, E.~Pimentel, and R.~de~Queiroz, editors, {\em
  Proc.\ of 28th {WoLLIC}}, pages 237--249, 2022.
\newblock {LNCS} 13468.
\newblock \href {https://doi.org/10.1007/978-3-031-15298-6\_15}
  {\path{https://doi.org/10.1007/978-3-031-15298-6\_15}}.

\bibitem{DitmarschGKP19}
H.~van Ditmarsch, M.~Gattinger, L.B. Kuijer, and P.~Pardo.
\newblock Strengthening gossip protocols using protocol-dependent knowledge.
\newblock {\em {FLAP}}, 6(1):157--203, 2019.
\newblock URL: \url{https://arxiv.org/abs/1907.12321}.

\bibitem{boukeLOFT}
H.~van Ditmarsch, D.~Grossi, A.~Herzig, W.~van~der Hoek, and L.B. Kuijer.
\newblock Parameters for epistemic gossip problems.
\newblock Proc.\ of 12th LOFT, 2016.

\bibitem{DitmarschKS17}
H.~van Ditmarsch, I.~Kokkinis, and A.~Stockmarr.
\newblock Reachability and expectation in gossiping.
\newblock In B.~An, A.~Bazzan, J.~Leite, S.~Villata, and L.~van~der Torre,
  editors, {\em Proc.\ of the 20th {PRIMA}}, LNCS 10621, pages 93--109.
  Springer, 2017.

\bibitem{logicofgossiping:2020}
H.~van Ditmarsch, W.~van~der Hoek, and L.B. Kuijer.
\newblock The logic of gossiping.
\newblock {\em Artificial Intelligence}, 286:103306, 2020.
\newblock \href {https://doi.org/10.1016/j.artint.2020.103306}
  {\path{https://doi.org/10.1016/j.artint.2020.103306}}.

\bibitem{DitmarschEPRS17}
H.~van Ditmarsch, J.~van Eijck, P.~Pardo, R.~Ramezanian, and
  F.~Schwarzentruber.
\newblock Epistemic protocols for dynamic gossip.
\newblock {\em J. Applied Logic}, 20:1--31, 2017.
\newblock \href {https://doi.org/10.1016/j.jal.2016.12.001}
  {\path{https://doi.org/10.1016/j.jal.2016.12.001}}.

\bibitem{hvdetal.dynamicgossip:2019}
H.~van Ditmarsch, J.~van Eijck, P.~Pardo, R.~Ramezanian, and
  F.~Schwarzentruber.
\newblock Dynamic gossip.
\newblock {\em Bulletin of the Iranian Mathematical Society}, 45(3):701--728,
  2019.
\newblock URL: \url{https://arxiv.org/abs/1511.00867}, \href
  {https://doi.org/10.1007/s41980-018-0160-4}
  {\path{https://doi.org/10.1007/s41980-018-0160-4}}.

\bibitem{west82a}
D.B. West.
\newblock A class of solutions to the gossip problem, part {I}.
\newblock {\em Discrete Mathematics}, 39(3):307--326, 1982.

\end{thebibliography}

\section*{Appendix A: Protocol-dependent knowledge is well-defined}\label{appendix:welldef}

We show that protocol-dependent knowledge $K^\prot_a \phi$ is well-defined. Define a relation $<$ as follows. For any call sequences $\sigma,\tau$, formulas $\phi,\psi$ and agents $a,b,c$:
\begin{enumerate}
\item $(\sigma,\phi) < (\tau,\psi)$ if $\phi$ is a subformula of $\psi$
\item $(\sigma, \prot_{ab}) < (\tau,K^\prot_c \phi)$ where $a \neq b$
\item $(\sigma,\top) < (\tau,\phi)$ where $\phi$ is not an atom
\item $(\sigma,S_a b) < (\tau,\phi)$ where $\phi$ is not an atom
\item $(\sigma,Cab) < (\tau,\phi)$ where $\phi$ is not an atom and $a \neq b$
\end{enumerate}
The relation $<$ is a well-founded partial order, with pairs (\emph{any call sequence}, \emph{any atom}) at the bottom. Recalling that $K^\prot_a \phi$ can be interpreted as $K_a (X,\phi)$ where $X = \{P_{bc} \mid b \neq c \in A \}$, clause 2.\ that $(\sigma, \prot_{ab}) < (\tau,K^\prot_c \phi)$ is already subsumed by clause 1., as $\prot_{ab}$ is then a subformula of $K^\prot_c \phi$.

We now show that the satisfaction relation $\models$ is well-defined using that relation $<$ is well-founded. We do this for the engaged agents semantics, without that it is even simpler. The proof is by structural induction. All clauses are trivial except knowledge. 

In order to determine $\sigma\models K^\prot_a \phi$, we need to determine for all $\tau$ such that $\tau \sim^\prot_a \sigma$ (where $\tau$ may be $\sigma$) that $\tau\models\phi$, as well as (for the engaged agents semantics) $\tau\models K^\prot_b\Exp_A$ or $\tau\models\neg K^\prot_b\Exp_A$ for agents $b$ possibly different from $a$.
\begin{itemize}
\item
Concerning $\tau\models\phi$, from clause $1.$ we obtain $(\tau,\phi) < (\sigma,K^\prot_a \phi)$.
\item
Concerning $\tau\models K^\prot_b\Exp_A$, this can be determined by checking that $\rho\models \Exp_A$ for any $\rho \sim^\prot_b \tau$. Determining $\rho \sim^\prot_b \tau$ introduces another obligation that will be honoured below. Now $\rho\models \Exp_A$ means that $\rho\models S_c d$ for any $c,d\in A$ (not necessarily different from $a$ or $b$). We then obtain from clause $4.$ that $(\rho,S_c d) < (\sigma,K^\prot_a \phi)$.  The case $\tau\models \neg K^\prot_b\Exp_A$ is treated similarly, first using that $(\tau,K^\prot_b\Exp_A) < (\tau, \neg K^\prot_b\Exp_A)$, by clause $1.$
\item
Concerning $\tau \sim^\prot_a \sigma$, this requires to establish $\tau' \models \prot_{cd}$ for $c,d \in A$ (where $c$ or $d$ may be $a$) and prefixes $\tau'$ of $\tau$. We now use clause $2.$ that $(\tau',\prot_{cd}) < (\sigma, K^\prot_a \phi)$. 

Similarly, concerning the novel obligation $\rho \sim^\prot_b \tau$ we need to establish $\rho' \models \prot_{cd}$ for prefixes $\rho'$ of $\rho$. Again, we use clause $2.$ to get $(\rho',\prot_{cd}) < (\sigma, K^\prot_a \phi)$. 
\end{itemize}

Note that it plays no role whether $\tau$ or $\rho$ are $\sim_a$ or $\sim_b$ related to $\sigma$ or even by some chain of such indistinguishability links. 

Further note that $\tau$ and $\rho$ may in length largely exceed $\sigma$ (and even may have $\sigma$ as a prefix themselves) given asynchrony. But this does not matter, the length of sequences does not play a role in the order (it is of some importance to observe this).

A particular case of clause $1.$ is when $\psi = [\tau]\phi$, such that for any call sequence $\tau$, $(\sigma;\tau,\phi) < (\sigma, [\tau]\phi)$.

\section*{Appendix B: Adding skip calls}\label{appendix:skip}

\subsection*{Syntax and semantics --- skip}

In this section we investigate how adding a $\Skip$ program to the language and semantics makes a difference in the termination of gossip protocols. We assume all prior enrichments of the semantics: known protocols and engaged agents. We will later see that our \Skip\ is different from the PDL-\Skip\ program defined as the test program $?\top$~\cite{hareletal:2000}. It rather is the \Skip\ featuring in some other publications on epistemic gossip~\cite{apt-tark,attamahetal.ecai:2014}, that should be seen as an explicit tick of the clock, during which no call is made. Given that it means absence of a call, such a \Skip\ program should not be named a \Skip\ {\bf call}. However, as we wish to continue to name call sequences to which \Skip\ programs have been added `call sequences', we stick to the term \Skip\ call.

We first change the program part of the BNF of the logical language to also take into account $\Skip$ calls. The relevant part of Definition~\ref{def.lang} was
\[ \begin{array}{lcl}
\pi & := & ?\phi \mid ab \mid (\pi ; \pi) \mid (\pi \union \pi) \mid \pi^*
\end{array} \]
and the new definition is:

\begin{definition}[Programs --- skip]
\[ \begin{array}{lcl}
\pi & := & ?\phi \mid \Skip \mid ab \mid (\pi ; \pi) \mid (\pi \union \pi) \mid \pi^*
\end{array} \]
where different $a,b$ range over $A$. 
\end{definition}
To allow $\Skip$ calls, we change the crucial Definition~\ref{def.prot} of protocol. Let us recall the original definition:
\[ \prot :=
  {(\Union_{a\neq b \in A} (?(\neg K_a^\prot \Exp_A \et \prot_{ab});ab))}^*;
  ? E^\prot\!\Exp_A \]
We now get:
\begin{definition}[Protocol --- skip]
\[ \begin{array}{lll} \prot &:=&
  {(\Union_{a\neq b \in A} (?(\neg K_a^\prot \Exp_A \et \prot_{ab});ab))}^* ; \\
 && ? \neg\Vel_{a \neq b \in A} (\neg K^\prot_a \Exp_A \et \prot_{ab}) ; \\
 && {(\Union_{a\neq b \in A}  (?(\neg K_a^\prot \Exp_A \et \neg\prot_{ab}); \Skip))}^* ; \\
&&  ? E^\prot\!\Exp_A \end{array} \]
where for all $a \neq b \in A$, $\prot_{ab} \in \lang$ is the protocol condition for call $ab$ of protocol $\prot$.
\end{definition}
Formula $\neg\Vel_{a \neq b \in A} (\neg K^\prot_a \Exp_A \et \prot_{ab})$ is the stop condition for the first arbitrary iteration. It is equivalent to  the more intuitive $\bigwedge_{a \neq b \in A}( \prot_{ab} \imp K^\prot_a \Exp_A)$, which we will use further below. Given its position in the program, we could replace the second arbitrary iteration ${(\Union_{a\neq b \in A}  (?(\neg K_a^\prot \Exp_A \et \neg\prot_{ab}); \Skip))}^*$ by the shorter ${(\Union_{a \in A}  (?\neg K_a^\prot \Exp_A; \Skip))}^*$ without changing the meaning of the protocol: the stop condition in the middle enforces that any agent satisfying $\neg K_a^\prot \Exp_A$ also satisfies $\neg\prot_{ab}$. We left the condition $\neg\prot_{ab}$ in place for intuitive clarity.

The second arbitrary iteration only fires if anyone satisfying the protocol condition is already a super expert, but when there still are agents who are not super experts (so that the protocol has not terminated super-successfully) but who do not satisfy the protocol condition.

We continue with the epistemic relations. Just as for the engaged agents semantics, the semantic relation $\models$ remains unchanged (Definition~\ref{Sema:Def}), we merely need to define the interpretation of program \Skip.

\begin{definition}[Epistemic relations and semantics of programs --- skip]\label{def.eprel} \ \\
Let $a\in A$. The synchronous accessibility relation $\synch_a^\prot$ between call sequences is the smallest symmetric and transitive relation satisfying all the clauses of Definition~\ref{def.synch2} plus the following two inductive clauses involving \Skip.
\begin{itemize}
 \item if $\sigma \synch_a^\prot \tau$,
 $a \notin \{b,c\}$,
 $\sigma \models \bigwedge_{d \neq e \in A}( \prot_{de} \imp K^\prot_d \Exp_A)$ and 
 $\tau  \models \neg K_b^\prot \Exp_A \et \prot_{bc}$,
 then
 $\sigma;\Skip \synch_a^\prot \tau;bc$
 \item if $\sigma \synch_a^\prot \tau$,
  $\sigma \models \bigwedge_{d \neq e \in A}( \prot_{de} \imp K_d^\prot \Exp_A)$ and
 $\tau  \models \bigwedge_{d \neq e \in A}( \prot_{de} \imp K_d^\prot \Exp_A)$,
 then
 $\sigma;\Skip \synch_a^\prot \tau;\Skip$
\end{itemize}
The asynchronous epistemic relation $\asynch_a^\prot$ is defined similarly, by adding the single clause:
\begin{itemize}
 \item if $\sigma \asynch_a^\prot \tau$ and $\sigma \models  \bigwedge_{c \neq d \in A}( \prot_{cd} \imp K^\prot_c \Exp_A)$
 then
 $\sigma;\Skip \asynch_a^\prot \tau$
\end{itemize}
To the semantics of programs (Definition~\ref{Sema:Def}) we add the interpretation of \Skip:
\[ \begin{array}{lcl}
\sigma\I{\Skip}\tau & \text{iff} & \tau = \sigma;\Skip
\end{array} \]
where $I^{\sigma;\Skip} := I^\sigma$.
\end{definition}

Note that \Skip\ calls can only occur at the postfix of a permitted call sequence. In other words, all call sequences $\sigma$ that are executions of protocols according to the \Skip\ semantics have shape $\sigma_1;\sigma_2$ where $\sigma_1$ only contains calls $ab$ for some $a,b \in A$, whereas $\sigma_2$ only contains \Skip\ calls. This also holds for infinite call sequences, i.e., an infinite call sequence may consist of calls $ab$ only, or of a finite prefix of such calls followed by an infinite postfix of \Skip\ calls.

Recalling the semantics of programs (Definition~\ref{Sema:Def}) we see that the PDL-\Skip\ defined as $?\top$ is defined as \[ \sigma\I{?\top}\tau \quad \text{iff} \quad \tau = \sigma. \] Note that this does not extend the call sequence, unlike our `clock tick' \Skip.

\emph{Skip} calls do not have factual consequences (changes of the value of atomic propositions): atoms $S_a b$ do not change value because $I^{\sigma;\Skip} = I^\sigma$, and atoms $Cab$ do not change value as \Skip\ is not a call $ab$. However, \Skip\ calls may have other informative consequences.

In the asynchronous semantics, \Skip\ calls do not have informative consequences. They go, so to speak, unnoticed. This is expressed by the following proposition.
\begin{proposition}
Assume asynchrony. Let call sequence $\sigma$ be given such that $\sigma \models  \bigwedge_{c \neq d \in A}( \prot_{cd} \imp K^\prot_c \Exp_A)$. Then $\sigma \models K_a^\prot \phi \eq [\Skip] K_a^\prot \phi$. 
\end{proposition}
\begin{proof}
First note that for any $\phi$ and $\sigma$: $\sigma \models [\Skip] \phi$, iff $\tau \models \phi$ for all $\tau$ such that $\sigma\I{\Skip}\tau$, iff $\sigma;\Skip \models \phi$. 

Let now $\phi \in \lang$ and call sequence $\sigma$ such that $\sigma \models  \bigwedge_{c \neq d \in A}( \prot_{cd} \imp K^\prot_c \Exp_A)$ be given. Then: $
\sigma \models K_a^\prot \phi$, iff $\tau \models \phi$ for all $\tau \asynch_a^\prot \sigma$, iff (*) $\tau \models \phi$ for all $\tau \asynch_a^\prot \sigma;\Skip$, iff $\sigma;\Skip \models K_a^\prot \phi$, iff $\sigma \models [\Skip]K_a^\prot \phi$. Therefore $\sigma \models K_a^\prot \phi \eq [\Skip] K_a^\prot \phi$.
 
In (*) we use that if $\sigma \models  \bigwedge_{c \neq d \in A}( \prot_{cd} \imp K^\prot_c \Exp_A)$, then from Definition~\ref{def.eprel} it follows that $\tau\asynch_a^\prot\sigma$ iff $\tau\asynch_a^\prot\sigma;\Skip$. 
\end{proof}
If \Skip\ calls can take place any time we even have $\models \phi \eq [\Skip]\phi$, as suggested by Wiebe van der Hoek in the context of~\cite{apt-tark,attamahetal.ecai:2014}. However, for our semantics only permitting \Skip\ when all agents are super experts, this is false. For example, given a super-successful sequence $\sigma$ for a protocol $\prot$, we have that $\sigma \models [\Skip]K^\prot_a\bot$, as \Skip\ is not permitted after termination. On the other hand, evidently, $\sigma \not\models K^\prot_a\bot$. So, $\sigma \not\models K^\prot_a\bot \eq [\Skip] K^\prot_a\bot$.

In the synchronous semantics, \Skip\ calls may have informative consequences, as we will now see. Because the agents become aware of time, this may result in knowledge gain.

\subsection*{Results for the protocol CMO}\label{section:results-skip-cmo}

\begin{theorem}
Synchronous known $\CMO$ with engaged agents and \Skip\ is super-successful.
\end{theorem}
\begin{proof}
Let $\sigma$ be a maximal $\CMO$-permitted sequence. Since $\CMO$ is successful, after executing $\sigma$ all agents are experts: $\Exp_A$ holds. If $E^\CMO\Exp_A$ now also holds, we are done. If $E^\CMO\Exp_A$ does not hold, then, since $\sigma$ is maximal, any agent who has not yet been involved in a call with some other agent, is already a super expert: $\Et_{b\neq c \in A} (\CMO_{bc} \imp K^\CMO_b \Exp_A)$. Also, since $\sigma$ is maximal but not super-successful, there is an agent $a$ who is not a super expert but who has been involved in a call with all other agents $\neg K^\CMO_a \Exp_A \et \Et_{b \in A} \neg \CMO_{ab}$.

Because $a$ is not a super expert, there is a call sequence $\tau$ such that $\sigma \synch_a \tau$ and $\tau \not \models \Exp_A$, i.e., there are $b,c \in A$ such that $\tau \not\models S_b c$. Therefore $\tau\not\models Cbc$ and $\tau \not\models Ccb$, so that $\tau \models \CMO_{bc}$. Protocol-dependent knowledge is truthful after the $\CMO$-permitted sequence $\tau$, therefore, from $\tau \not \models \Exp_A$ it also follows that $\tau \not \models K^\CMO_b \Exp_A$.

From all this it therefore follows that $\sigma;\Skip \synch_a \tau;bc$. If we now have that $\sigma;\Skip \models E^\CMO\Exp_A$, we are done. Otherwise, we repeat the procedure until the maximum number $\binom{n}{2}$ of $\CMO$-permitted calls has been reached. After that, $E^\CMO\Exp_A$ is a property of that horizon.
\end{proof}

\begin{example}
We recall Figure~\ref{cmo.figure}, Theorem~\ref{thm:syn-k-cmo-sups}, and Theorem~\ref{thm:ck-and-sync-CMO}. Synchronous known $\CMO$ is super-successful, however with engaged agents it is not.

Reconsider $\sigma = ab;bc;cd;ad;bd$ and $\tau = ab;bc;ac;ad;bd$, and recall that $\sigma \synch_{b} \tau$. After $\sigma$ all agents are experts. Agent $b$ does not know that, because $b$ considers $\tau$ possible. Call $ac$ is not $\CMO$-permitted after $\sigma$, because $a$ is a super expert. After $\tau$ agent $c$ does not know the secret of $d$ and so $cd$ is $\CMO$-permitted. We now have that $\sigma;\Skip \synch_{b} \tau;cd$ and $\sigma;\Skip \models E^{\CMO}\Exp_A$. 
\end{example}

\section*{Appendix C: GoMoChe Scripts}\label{appendix:gomoche}

Here we list the commands which can be used to verify some of the examples mentioned throughout the text with GoMoChe, the model checker for gossip protocols available at \url{https://github.com/m4lvin/GoMoChe}.
Each command below should be run in GHCi, the interactive Haskell compiler.
For more details how to compile and use GoMoChe we refer to its readme.

GoMoChe uses Dynamic Gossip~\cite{DitmarschEPRS17,hvdetal.dynamicgossip:2019} where also phone numbers are exchanged.
But here we are interested in static gossip where agents only exchange secrets.
We also assume that everyone can call everyone from the beginning.
Fortunately the second assumption means we do not have to worry about the first one.
We can simply use the \texttt{totalInit} function from GoMoChe in our queries below.

Also note that GoMoChe usually assumes synchrony.
An experimental version of the program to deal with asynchrony is available in the `\texttt{async}' branch under \url{https://github.com/m4lvin/GoMoChe/tree/async}.
Most commands below are for the synchronous version, unless \texttt{ASync} is explicitly mentioned.

\medskip

After these preliminary explanations, we now continue with the scripts verifying the examples.

\medskip

\noindent
{\bf Example~\ref{ex.threea} (page~\pageref{ex.threea}):} 

\noindent Find the shortest super-successful sequence for three agents using asynchronous ANY.

\medskip

\begin{lstlisting}
λ> let myState = (ASync, totalInit 3, [])
λ> charSequence $ head $ concat [ filter (isSuperSuccSequence anyCall myState)
                                         (sequencesUpTo anyCall myState n)
                                | n <- [1..] ]
"ab;ac;ab;bc"
\end{lstlisting}

\noindent
This uses the following function included in GoMoChe:

\medskip

\begin{lstlisting}
isSuperSuccSequence :: Protocol -> State -> Sequence -> Bool
isSuperSuccSequence proto (g,sigma) cs =
  isSuccSequence (m,g,sigma) cs &&
  (g, sigma ++ cs) |= ForallAg (`superExpert` proto)
\end{lstlisting}

\noindent
Note that \texttt{|=} is $\models$ and that \texttt{superExpert} takes a protocol as a second argument.
This is the protocol which is assumed to be common knowledge as discussed in Section~\ref{section:ComKnowP}.

\medskip

\noindent
{\bf Example~\ref{ex.fourany} (page~\pageref{ex.fourany}):} 

\noindent We check that $ab;cd;ac;ad;bc;ba;bd$ is super-successful:

\medskip

\begin{lstlisting}
λ> isSuperSuccSequence anyCall
                       (totalInit 4, [])
                       (parseSequence "ab;cd;ac;ad;bc;ba;bd")
True
\end{lstlisting}

\medskip

\noindent
Example~\ref{ex.fourany} involves generating Table~\ref{tab:sevenCalls}  (page~\pageref{tab:sevenCalls}) for the sequence $ab;cd;ac;ad;bc;ba;bd$ with $\ANY$:

\medskip

\begin{lstlisting}
λ> knowledgeOverview (totalInit 4, parseSequence "ab;cd;ac;ad;bc;ba;bd") anyCall
    a           b           c           d     
ab  ab         ab           c           d     
cd  ab         ab           cd         cd     
ac  abcd A C   ab         abcd A C     cd     
ad  abcd A CD  ab         abcd A C   abcd A CD
bc  abcd A CD  abcd  BC   abcd ABCD  abcd A CD
ba  abcd ABCD  abcd ABC   abcd ABCD  abcd A CD
bd  abcd ABCD  abcd ABCD  abcd ABCD  abcd ABCD
\end{lstlisting}

\noindent Further to Example~\ref{ex.fourany}, we also check that with asynchronous $\ANY$ all call sequences of at most $7$ calls are not super-successful.
To restrict the search space, w.l.o.g.\ we fix the first two calls to be $ab;bc$ or $ab;cd$.
Moreover, also w.l.o.g., we use a version of $\ANY$ that only allows one of the two calls $xy$ or $yx$ for each pair of agent.
Note that these queries still take multiple minutes.

\medskip

\begin{lstlisting}
λ> :set +s
λ> all (\sigma -> (ASync, totalInit 4, [(0,1),(1,2)] ++ sigma)
                  |= Neg (allSuperExperts (wlog anyCall)))
       (sequencesUpTo (wlog anyCall) (ASync, totalInit 4, [(0,1),(1,2)]) 5)
True
(169.25 secs, 735,797,981,656 bytes)
λ> all (\sigma -> (ASync, totalInit 4, [(0,1),(2,3)] ++ sigma)
                  |= Neg (allSuperExperts (wlog anyCall)))
       (sequencesUpTo (wlog anyCall) (ASync, totalInit 4, [(0,1),(2,3)]) 5)
True
(386.56 secs, 1,664,502,795,960 bytes)
\end{lstlisting}

\medskip

\noindent
{\bf Example~\ref{ex.6Calls} (page~\pageref{ex.6Calls}):} 

\noindent Search for a sequence where $a$ is a super expert after the last call, but was not even an expert just before this last call.

\medskip

\begin{lstlisting}
λ> let myState = (totalInit 4, [])
λ> let myfilter = filter (\sigma ->
     ( (totalInit 4,      sigma) |= superExpert 0 anyCall )
  && ( (totalInit 4, init sigma) |= Neg (expert 0)        ) )
λ> charSequence $ head $ myfilter (sequencesUpTo anyCall myState 6)
"ab;ac;cd;ab;bc;ab"
\end{lstlisting}

\noindent Further to Example~\ref{ex.6Calls}, we generate Table~\ref{tab:6Calls} (page~\pageref{tab:6Calls}) for the sequence $ab;ac;cd;ab;bc;ab$ with $\ANY$:

\medskip

\begin{lstlisting}
λ> knowledgeOverview (totalInit 4, parseSequence "ab;ac;cd;ab;bc;ab") anyCall
    a           b           c           d     
ab  ab         ab           c           d     
ac  abc        ab         abc           d     
cd  abc        ab         abcd   CD  abcd   CD
ab  abc    CD  abc        abcd   CD  abcd   CD
bc  abc    CD  abcd  BCD  abcd  BCD  abcd   CD
ab  abcd ABCD  abcd ABCD  abcd  BCD  abcd   CD
\end{lstlisting}

\noindent Finally, in Example~\ref{ex.6Calls} we check that there is no sequence of at most 7 calls which is super-successful in the asynchronous setting.
Again we separately consider the two cases of overlap or no overlap in the first two calls.

\medskip

\begin{lstlisting}
λ> :set +s
λ> let myState = (ASync, totalInit 4, [(0,1),(1,2)])
λ> print [ (n+2, filter (isSuperSuccSequence (wlog anyCall) myState)
                        (sequencesUpTo (wlog anyCall) myState n)) | n <- [1..5] ]
[(3,[]),(4,[]),(5,[]),(6,[]),(7,[])]
(155.12 secs, 520,606,884,592 bytes)
λ> let myState = (ASync, totalInit 4, [(0,1),(2,3)])
λ> print [ (n+2, filter (isSuperSuccSequence (wlog anyCall) myState)
                        (sequencesUpTo (wlog anyCall) myState n)) | n <- [1..5] ]
[(3,[]),(4,[]),(5,[]),(6,[]),(7,[])]
(450.96 secs, 1,292,209,208,208 bytes)
\end{lstlisting}

\medskip

\noindent
{\bf Generating Table~\ref{tab:fiveCalls} (page~\pageref{tab:fiveCalls})} for the sequence $ab;cd;bd;ac;bc$ with known CMO:

\medskip

\begin{lstlisting}
λ> knowledgeOverview (totalInit 4, parseSequence "ab;cd;bd;ac;bc") cmo
    a           b           c           d     
ab  ab         ab           c           d     
cd  ab         ab           cd         cd     
bd  ab         abcd  B D    cd    D  abcd  B D
ac  abcd ABCD  abcd AB D  abcd ABCD  abcd  BCD
bc  abcd ABCD  abcd ABCD  abcd ABCD  abcd ABCD
\end{lstlisting}

\noindent
{\bf Generating Table~\ref{tab:sixCalls} (page~\pageref{tab:sixCalls})} for the sequence $ab;bc;cd;ad;bd;ac$ with $\CMO$:

\medskip

\begin{lstlisting}
λ> knowledgeOverview (totalInit 4, parseSequence "ab;bc;cd;ad;bd;ac") cmo
    a           b           c           d     
ab  ab         ab           c           d     
bc  ab         abc        abc           d     
cd  ab         abc        abcd   CD  abcd   CD
ad  abcd A  D  abc        abcd   CD  abcd A CD
bd  abcd ABCD  abcd  B D  abcd ABCD  abcd ABCD
ac  abcd ABCD  abcd ABCD  abcd ABCD  abcd ABCD
\end{lstlisting}

\end{document}